\documentclass{article}
\usepackage[preprint]{neurips_2024}

\usepackage[utf8]{inputenc}
\usepackage[T1]{fontenc}
\usepackage{amsmath,amssymb,amsfonts,amsthm}
\usepackage{booktabs}
\usepackage{graphicx}
\usepackage{mathtools}
\usepackage{microtype}
\usepackage{natbib}
\usepackage{subcaption}
\usepackage{array}
\usepackage{url}
\usepackage{xcolor}
\usepackage{hyperref}
\usepackage{fancyhdr}

\newtheorem{theorem}{Theorem}[section]
\newtheorem{corollary}[theorem]{Corollary}
\newtheorem{proposition}[theorem]{Proposition}
\newtheorem{lemma}[theorem]{Lemma}

\theoremstyle{definition}
\newtheorem{definition}{Definition}
\newtheorem{assumption}{Assumption}
\theoremstyle{remark}

\newcommand{\E}{\mathbb E}
\newcommand{\ind}{\mathbf 1}
\newcommand{\PiC}{\mathcal P}
\newcommand{\Ghat}{\widehat{\mathcal G}}
\newcommand{\Shat}{\widehat{\mathcal S}}

\title{Support-aware offline policy selection for advertising marketplaces}
\author{\fontsize{11}{22}\selectfont Prashant Shekhar\thanks{Corresponding author} and Caroline Howard \\ \textit{\fontsize{10}{22}\selectfont Department of Mathematics} \\ \textit{\fontsize{10}{22}\selectfont Embry-Riddle Aeronautical University, Daytona Beach, FL, USA}}
\begin{document}
\maketitle

\begin{abstract}
Logged advertising auctions make offline reserve-price evaluation attractive but risky. Replay tables can identify policies with large apparent yield gains, yet they can also hide weak threshold support, multiple-comparison effects, subgroup harm, and bidder-response uncertainty. Existing replay and off-policy evaluation methods estimate or rank policy values, but they do not directly answer the operational question of whether the available evidence is strong enough to justify validation. This paper develops a support-aware offline decision framework for reserve-policy selection. Rather than outputting a single point-estimate winner, the framework converts logged evidence into a conservative decision object consisting of certified policies, statistically dominated alternatives, and unresolved candidates requiring further validation. The main theoretical result gives a unified finite-catalog guarantee showing that, under simultaneous uncertainty control and conservative support gates, the framework preserves the best gate-passing policy while eliminating only policies with certified regret. Supporting results characterize support-localized replay generalization, establish information-theoretic threshold-resolution limits, and quantify when heterogeneous bidder response can overturn localized replay rankings. Experiments on iPinYou real-time-bidding logs show that the leading reserve rule achieves a 47.66\% replay lift in season two, a 40.71\% simultaneous lower-bound lift, and a 43.87\% frozen out-of-time replay lift in season three. The framework reduces a 19-policy catalog to a two-policy validation shortlist while certifying non-harm across 44 advertiser, exchange, and region segments. The results support the central claim that offline reserve-policy evaluation should produce certified validation decisions rather than point-estimate rankings alone.
\end{abstract}

\section{Introduction}

Advertising marketplaces continuously choose reserve or floor prices that affect revenue, fill, bidder incentives, and marketplace liquidity. These decisions are natural candidates for offline policy learning because platforms log auction opportunities, bids, floors, payments, and fill outcomes. Given such logs, an analyst can ask what would have happened if the platform had imposed a different floor on the same opportunities. This replay question is operationally useful, but it is also easy to overinterpret. A candidate policy can look attractive because it raises payments on retained impressions while the logs contain little support near the new clearing threshold. It can improve aggregate metrics while harming a particular exchange, region, or advertiser group. It can also survive a static replay exercise even though real bidders may later respond through budgets, pacing, or bid shading. The central difficulty is therefore not only estimating offline lift, but determining whether the logged evidence is strong enough to justify validation.

This paper takes the position that offline reserve-price analysis should not output a single winner ranked by point estimate. It should instead produce a conservative decision object, where some policies are sufficiently supported to justify validation, some are statistically dominated and can be removed, while others have apparent upside but remain unresolved because the logged data do not support a stronger claim. This distinction is important in advertising markets, where reserve-price changes can affect bidder participation and where online validation capacity is costly. In this regard, we develop a support-aware decision framework specialized to finite reserve-policy catalogs. Here finite catalogs are not merely a technical convenience, but a faithful representation of marketplace operations, where reserve-price candidates usually arise from product constraints, revenue-management rules, risk limits, and implementation review. The relevant operational problem is therefore often not unrestricted functional optimization, but screening a realistic catalog of candidate policies. The proposed framework exploits this structure while making the statistical cost of catalog screening explicit.

The theoretical novelty of the proposed framework lies in the composed decision guarantee rather than in any single concentration inequality. Standard replay analyses can estimate or rank policy values, but they do not by themselves answer the operational question of which policies are certified, which are dominated, and which must remain unresolved after support, multiplicity, and subgroup-safety gates are all applied. Figure~\ref{fig:framework-overview} illustrates this central problem. Logged advertising auctions make offline replay evaluation feasible, but naive replay rankings can be misleading because apparent reserve-policy gains may rely on weak threshold support, multiple-comparison effects, subgroup harm, or unstable bidder response.

\begin{figure}[t]
\centering
\includegraphics[width=0.88\linewidth]{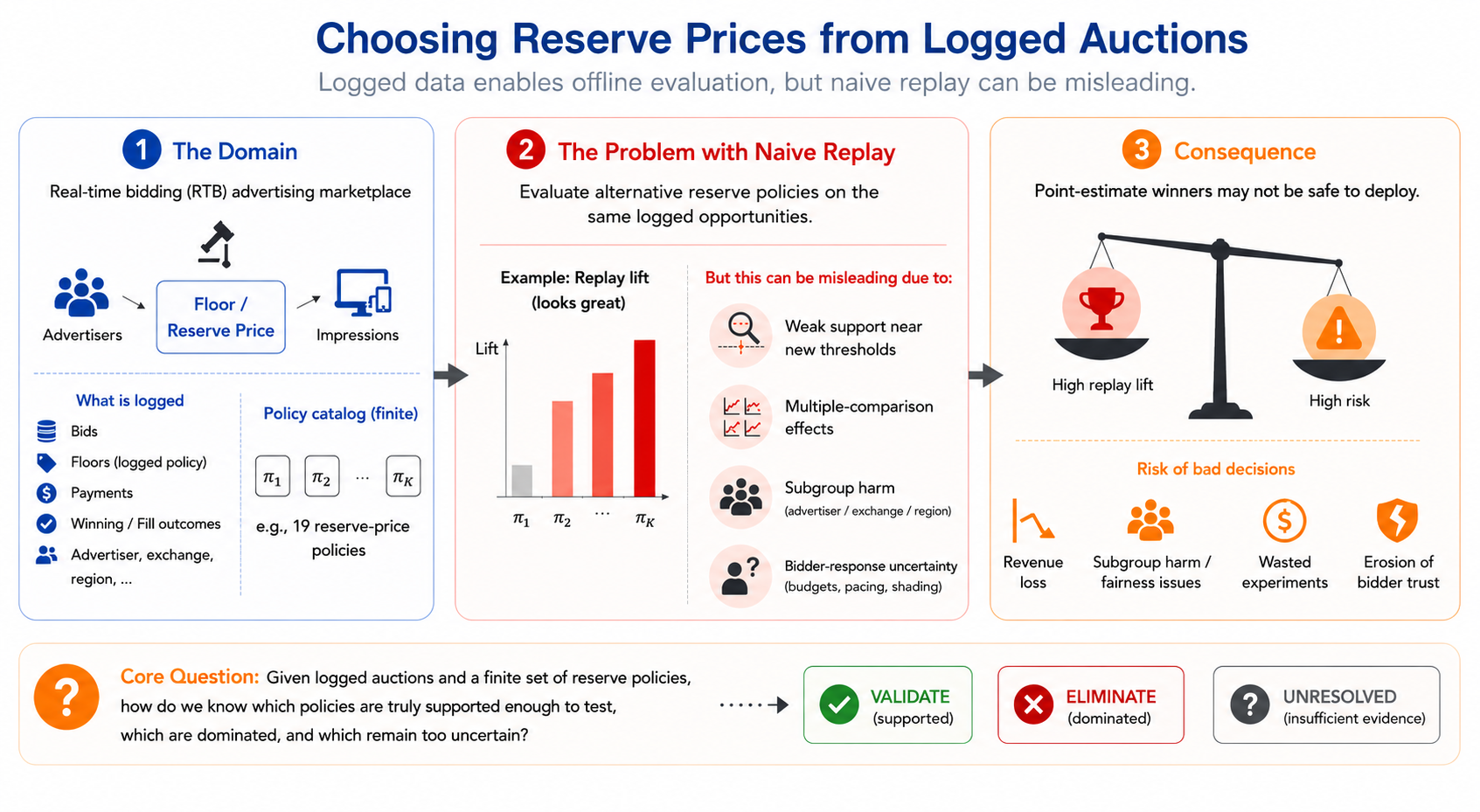}
\caption{\small Offline reserve-policy selection from logged advertising auctions. Logged marketplace data and a finite reserve-policy catalog make offline replay evaluation possible, but naive replay rankings can be misleading because apparent gains may hide weak threshold support, multiple-comparison effects, subgroup harm, or bidder-response uncertainty. The figure illustrates the central decision problem of determining which reserve policies are sufficiently supported by logged evidence to justify validation, which are statistically dominated, and which remain unresolved because the available evidence is insufficient for a stronger claim.}
\label{fig:framework-overview}
\end{figure}

Overall, the paper makes five broad contributions. \textbf{First}, it formulates offline reserve-price selection as a decision-support problem whose output is a conservative validation shortlist together with dominated and unresolved alternatives. \textbf{Second}, it gives a unified decision-pipeline guarantee that jointly controls multiple comparisons, support gates, and subgroup safety. \textbf{Third}, it derives support-localized replay guarantees and characterizes when heterogeneous bidder response can overturn localized replay rankings. \textbf{Fourth}, it establishes information-theoretic threshold-resolution limits governed by local boundary support near reserve thresholds. \textbf{Fifth}, it converts subgroup diagnostics into formal non-harm certificates and validates the resulting framework on iPinYou real-time-bidding logs with frozen out-of-time replay.

\section{Related Work}

\noindent \textbf{Reserve prices and advertising auctions.}
Classical auction theory establishes reserve prices as revenue-relevant design variables \citep{myerson1981optimal}. Sponsored-search and advertising auction models explain why platform pricing, ranking, and auction format interact with bidder incentives \citep{edelman2007internet,varian2007position}. Field and empirical studies show that reserve prices can materially change advertising-auction revenue \citep{ostrovsky2011reserve,yuan2014reserve}. The iPinYou benchmark provides public real-time-bidding logs for reproducible advertising-marketplace analysis \citep{liao2014ipinyou}. Learning-theoretic work studies reserve optimization and revenue maximization in auction settings \citep{mohri2014learning,cole2014sample,feng2020reserve}. In a nutshell, prior work focuses primarily on revenue optimization or reserve estimation, whereas this paper studies the decision problem of determining which reserve policies are sufficiently supported by logged evidence to justify validation.

\noindent \textbf{Counterfactual learning from logs.}
Computational advertising has long motivated counterfactual reasoning from production logs \citep{bottou2013counterfactual}. Logged-bandit and off-policy evaluation work provide inverse-propensity, doubly robust, and variance-aware estimators \citep{dudik2011doubly,swaminathan2015counterfactual}. This paper uses replay as a statistical component inside a broader decision framework. The central object is a conservative policy-screening rule that distinguishes certified, dominated, and unresolved policies under support and uncertainty constraints.

\noindent \textbf{Marketplace experiments and interference.}
Marketplace tests can violate simple randomization assumptions because treated and control units interact through shared supply, budgets, ranking, and congestion \citep{johari2022experimental,li2022interference,bajari2023marketplace}. Switchback and cluster designs are common responses to these issues \citep{bojinov2022switchback,holtz2020interference,bright2022shadow}. This paper therefore stops short of claiming deployment safety from replay evidence alone. Offline replay identifies validation targets, while live causal safety still requires randomized experimentation under marketplace interference.

\noindent \textbf{Decision support.}
Decision-support systems emphasize that predictive or statistical scores must be embedded in a decision process that exposes uncertainty, constraints, and operational consequences \citep{coussement2021interpretable}. A closely related applied framework is developed by \citet{shekhar2026decision}, who study replay-to-launch readiness for marketplace floor policies under incomplete evidence. The present paper differs by developing formal guarantees for conservative policy certification, elimination, and subgroup safety under finite logged support before any live launch decision is made.

\section{Problem Setup}

Let \(i=1,\ldots,n\) index logged auction opportunities. The platform observes context \(X_i\), logged floor \(f_i^0\), highest logged bid \(b_i\), logged payment \(p_i\) conditional on a filled impression, and logged fill indicator \(D_i^0\in\{0,1\}\). A candidate reserve policy \(\pi\) maps \(X_i\) to a counterfactual reserve \(f_i^\pi\). The baseline policy \(\pi_0\) satisfies \(f_i^{\pi_0}=f_i^0\). We focus on a finite catalog
$
\PiC=\{\pi_0,\pi_1,\ldots,\pi_M\}
$
chosen before evaluation. The static replay outcome under policy \(\pi\) is
$
Y_i^\pi
=
\ind\{D_i^0=1,\; b_i\ge f_i^\pi\}\max\{p_i,f_i^\pi\}.
$
The replay contract therefore retains only impressions that were originally filled and whose logged bids clear the counterfactual reserve. The replay lift relative to the logged baseline is
$
\Delta_\pi
=
\frac{\mu_\pi-\mu_0}{\mu_0}.
$
where $\mu_\pi = \E[Y_i^\pi]$ and  $\mu_0 = \E[Y_i^{\pi_0}]$. This replay estimand holds the opportunity set, bids, and bidder participation fixed. It is useful for offline screening but is not a live causal effect.

\begin{assumption}[Fixed-bid replay]
Counterfactual replay applies candidate reserve policies to the logged opportunity set while holding bids, participation, and pacing decisions fixed.
\end{assumption}

\begin{assumption}[Monotone reserve candidates]
Candidate reserve policies in the catalog satisfy
$
f_i^\pi\ge f_i^0
$
$\forall i.
$
The monotonicity restriction reflects the practical focus on conservative reserve increases relative to the production baseline.
\end{assumption}

\begin{assumption}[Payment consistency]
If a logged filled impression is retained under \(f_i^\pi\), the retained payment equals
$
\max\{p_i,f_i^\pi\}.
$
\end{assumption}

\begin{definition}[Offline decision object]
An offline reserve-policy analysis returns three sets consisting of certified policies, dominated policies, and unresolved policies. Certified policies have positive lower-bound evidence after support and safety gates. Dominated policies have upper bounds below a retained lower bound. Unresolved policies have insufficient evidence for either certification or elimination and therefore require additional validation. More details on this are provided in the next section.
\end{definition}

\section{Main Results}

This section formalizes the paper's central claim that offline reserve-policy evaluation should produce a conservative decision object rather than a single point-estimate winner. The analysis proceeds in \textbf{four} stages followed by a final composition stage. First, localized replay learning is studied under a fixed-bid estimand, showing the significance of localized analysis and how it can provide independent information about policy ranking. Second, threshold-resolution results characterize when nearby reserve policies become statistically indistinguishable because the logged marketplace contains too little mass near the decision boundary. Third, lower-bound ranking converts simultaneous uncertainty into a regret-certified shortlist that separates defensible validation targets from statistically dominated alternatives. Fourth, segment-level non-harm results characterize when subgroup diagnostics can certify marketplace safety over a covered segment space. Finally, the offline decision pipeline theorem combines these components into a single guarantee.

The developed results are designed to be operational rather than purely asymptotic. Each theorem or corollary corresponds to a component of the decision pipeline: (i) localized replay screening, (ii) support diagnostics, (iii) shortlist construction, or (iv) subgroup certification. Taken together, the results characterize what offline replay evidence can certify, what it can eliminate, and what remains unresolved without further online validation.

Throughout the section, replay analysis is interpreted relative to the fixed-bid replay estimand introduced in Section~3. The results therefore establish statistical guarantees for the logged bidding environment, not causal guarantees for the live marketplace after bidder adaptation, pacing changes, or interference effects.

\subsection{Replay learning under fixed bids}

Replay evaluation is meaningful only relative to a fixed-bid estimand. The replay table answers the counterfactual question:
\emph{what would have happened if the platform had applied a different reserve rule to the same logged auction opportunities and bids?}
The resulting object is useful for offline screening, but it is not a live causal effect because bidder participation, pacing, and bidding strategies are held fixed.

Reserve policies differ primarily near threshold boundaries. Let
$
G_i=b_i-f_i^0
$
denote the logged bid-floor gap. A policy that changes the reserve threshold by \(\tau_\pi\) modifies replay outcomes mainly for auctions near the boundary
$
|G_i-\tau_\pi|\approx 0.
$
The relevant notion of statistical support is therefore local rather than global. The useful diagnostic is not only how many auctions were logged, but how tightly the logged bid-floor gaps concentrate around the reserve thresholds where candidate policies differ. The result below captures this idea by fixing a target boundary mass \(q\) and asking how large a neighborhood is needed to collect that mass.

The decomposition used in the theorem separates the replay difference
$
Z_i^\pi=Y_i^\pi-Y_i^0
$
into a boundary-supported component and a residual component. The boundary-supported term \(Z_{i,q}^\pi\) contains the part of the replay contrast that is estimated from auctions close to the policy threshold. The residual \(R_{i,q}^\pi\) collects the remaining contribution to the population replay value. This does not assume that all auctions outside the local band are irrelevant. Instead, it formalizes the idea that threshold policies are most sensitive near their decision boundary, while the mean contribution left outside the selected local band should be controlled by the width of the band needed to collect enough evidence. We therefore impose the weaker mean condition
$
\bigl|\mathbb{E}[R_{i,q}^\pi]\bigr|\le L_\pi r_\pi(q),
$
which is a local smoothness or approximation-bias condition. It is weaker than requiring every auction-level residual to be bounded by \(L_\pi r_\pi(q)\), and it is appropriate because replay lift is a mean policy-value object.

While aggregate replay remains the paper's primary estimate of total fixed-bid policy value, the theorem below introduced the localized quantity \(\widehat{\Delta}_{\pi,q}\), which serves a different purpose. It isolates the part of the replay contrast that is most directly supported near the reserve threshold where the policy changes auction outcomes. This matters because two policies can have similar aggregate replay lifts while relying on very different local evidence. This \(q\)-localized diagnostic therefore asks whether a policy's advantage is backed by observations close to its decision boundary, rather than by broad replay arithmetic alone. Appendix~\ref{app:q-localized} reports this diagnostic empirically and shows that localized ranking can differ from aggregate replay ranking, which is why the framework treats localized support as a separate decision signal rather than as a replacement for full replay.

\begin{theorem}[Support-localized replay generalization]
\label{thm:replay-generalization}
Let \(\PiC\) be a finite catalog of reserve policies. Fix \(q\in(0,1)\). For each policy \(\pi\), define the \(q\)-local boundary radius
$
r_\pi(q)
=
\inf\left\{
r\ge 0:
\Pr\left(|G_i-\tau_\pi|\le r\right)\ge q
\right\},
$
and let
\[
A_{\pi,q}
=
\{|G_i-\tau_\pi|\le r_\pi(q)\},
\qquad
m_{\pi,q}=\Pr(A_{\pi,q}).
\]
When the distribution of \(G_i\) has no atom at the boundary of \(A_{\pi,q}\), \(m_{\pi,q}=q\); otherwise \(m_{\pi,q}\ge q\). Define
$
\Delta_\pi=\frac{\mu_\pi-\mu_0}{\mu_0},
$
with \(0\le Y_i^\pi\le B\) almost surely and \(\mu_0>0\). Suppose that for each policy \(\pi\), the replay difference
$
Z_i^\pi := Y_i^\pi-Y_i^0
$
can be written as
$
Z_i^\pi = Z_{i,q}^\pi + R_{i,q}^\pi,
$
where \(Z_{i,q}^\pi\) is supported only on \(A_{\pi,q}\), \(|Z_{i,q}^\pi|\le B\), and the residual obeys the localization envelope
$
\bigl|\mathbb{E}[R_{i,q}^\pi]\bigr|\le L_\pi r_\pi(q).
$
Define the \(q\)-localized replay estimator
$
\widehat{\Delta}_{\pi,q}
=
\frac{1}{n\mu_0}\sum_{i=1}^n Z_{i,q}^\pi.
$
Then there exists a universal constant \(C>0\) such that, with probability at least \(1-\delta\),
\[
\sup_{\pi\in\PiC}
|\widehat{\Delta}_{\pi,q}-\Delta_\pi|
\le
\sup_{\pi\in\PiC}
\left[
C\frac{B}{\mu_0}
\sqrt{\frac{m_{\pi,q}\log(2|\PiC|/\delta)}{n}}
+
C\frac{B}{\mu_0}
\frac{\log(2|\PiC|/\delta)}{n}
+
\frac{L_\pi r_\pi(q)}{\mu_0}
\right].
\]
\end{theorem}

The theorem formalizes a quantile-local view of replay learning. The analyst fixes the amount of boundary evidence to be trusted, represented by \(q\), and evaluates the corresponding localized replay score \(\widehat{\Delta}_{\pi,q}\). The data then determine \(r_\pi(q)\), the smallest radius needed to collect that evidence around the policy threshold. If logged bid-floor gaps are concentrated near the threshold, \(r_\pi(q)\) is small and the localization term is tight. If the same amount of evidence can be collected only by using a wide neighborhood, the bound becomes looser because the policy comparison is being supported by observations farther from the threshold where the local approximation is less precise. The object being controlled is deliberately the localized estimator, not the full empirical replay estimator. Under a mean residual condition, the residual contributes approximation bias; controlling the full empirical residual would require an additional variance or envelope assumption.

In Theorem~\ref{thm:replay-generalization} the bound contains three interpretable components. The stochastic concentration term
$
C\frac{B}{\mu_0}
\sqrt{\frac{m_{\pi,q}\log(2|\PiC|/\delta)}{n}}
$
captures sampling variation in the selected boundary band. For continuous bid-gap distributions, \(m_{\pi,q}=q\), so this term is controlled by the chosen local evidence fraction rather than by the full panel size alone. The second-order term
$
C\frac{B}{\mu_0}
\frac{\log(2|\PiC|/\delta)}{n}
$
is a lower-order finite-sample correction that vanishes more quickly as \(n\) grows. The localization term
$
\frac{L_\pi r_\pi(q)}{\mu_0}
$
measures the mean residual approximation error induced by having to look away from the threshold to collect the required evidence. The theorem therefore has the desired monotonic intuition. For a fixed evidence fraction \(q\), a smaller required radius means stronger local support and a tighter guarantee, while a larger required radius means weaker local concentration and a looser guarantee.

The finite-catalog dependence is logarithmic in \(|\PiC|\), which matches marketplace practice, where teams typically compare a restricted set of implementable reserve rules rather than optimize over unrestricted policy classes. The replay-regret interpretation is operationally important. Corollary~\ref{cor:replay-regret} shows that the localized replay maximizer is controlled not only by the total log size, but also by support and residual terms around the chosen localization band. Policies supported by sparse boundary regions can therefore incur substantially larger replay regret even in very large auction logs. The result reinforces one of the central themes of the paper, namely that replay reliability is fundamentally governed by local threshold evidence rather than by global sample size alone.

Thus, the \(q\)-localized estimator should be interpreted as a support-aware diagnostic for threshold credibility. It is useful for identifying policies whose local replay evidence is strong, while aggregate replay remains necessary for measuring total marketplace value.

\begin{corollary}[Localized replay ranking under bounded response]
\label{cor:support-response}
Let
$
V_\pi = R_\pi + \Gamma_\pi
$
be the live policy value, where \(R_\pi\) is the fixed-bid replay value and \(\Gamma_\pi\) is the marketplace response term. From Theorem~\ref{thm:replay-generalization} define
$
\epsilon_\pi(q)
=
C\frac{B}{\mu_0}\sqrt{\frac{m_{\pi,q}\log(2|\PiC|/\delta)}{n}}
+
C\frac{B}{\mu_0}\frac{\log(2|\PiC|/\delta)}{n}
+
\frac{L_\pi r_\pi(q)}{\mu_0}.
$
Then, with probability at least \(1-\delta\),
$
V_\pi>V_{\pi'}
$
whenever
$
\widehat{\Delta}_{\pi,q}-\widehat{\Delta}_{\pi',q}
>
\epsilon_\pi(q)+\epsilon_{\pi'}(q)
+
\frac{\Gamma_{\pi'}-\Gamma_\pi}{\mu_0}.
$
In particular, if
$
|\Gamma_\pi-\Gamma_{\pi'}|\le \eta,
$
then the localized replay ranking is certified whenever
$
\widehat{\Delta}_{\pi,q}-\widehat{\Delta}_{\pi',q}
>
\epsilon_\pi(q)+\epsilon_{\pi'}(q)
+
\frac{\eta}{\mu_0}.
$
\end{corollary}

The corollary separates two distinct sources of replay instability. The first is statistical uncertainty from finite support near reserve thresholds, represented by the support-localized terms \(\epsilon_\pi(q)\). The second is marketplace response uncertainty, represented by the response gap \(\Gamma_{\pi'}-\Gamma_\pi\). The margin in the display is deliberately the localized replay margin \(\widehat{\Delta}_{\pi,q}-\widehat{\Delta}_{\pi',q}\), because Theorem~\ref{thm:replay-generalization} controls the localized estimator rather than the full empirical replay table. Operationally, the result says that a localized replay ranking is credible only when its margin is large enough to absorb both local statistical error and possible bidder response. This gives a validation criterion rather than a launch guarantee. Large localized replay margins with strong local support are more likely to survive online experimentation, whereas small margins or weakly supported thresholds can be overturned either by estimation noise or by heterogeneous bidder response after deployment.

\subsection{Support-limited reserve thresholds}

Reserve policies often differ only through small changes in floor thresholds. As before, let
$
G_i=b_i-f_i^0
$
denote the logged bid-floor gap, and consider two threshold policies \(\pi_\tau\) and \(\pi_{\tau'}\) separated by distance
\[
d_{\tau,\tau'}=|\tau-\tau'|.
\]
This distance measures how far apart the two reserve thresholds are and determines the width of the boundary region in which the policies make different replay decisions. The two policies differ only for auctions whose logged bid-floor gaps fall inside the separating boundary region
$
A_{\tau,\tau'}
=
\left\{
\min(\tau,\tau')
\le G_i \le
\max(\tau,\tau')
\right\}.
$ Building on this, we define the corresponding pairwise boundary mass by
$
m_{\tau,\tau'}
=
\Pr(A_{\tau,\tau'}).
$
This quantity determines the effective support available for distinguishing the two reserve rules. Even when the full replay log is large, only observations inside the separating boundary region contribute information about which policy is better. The relevant sample size is therefore the effective boundary sample size \(n\,m_{\tau,\tau'}\), not the full panel size \(n\).

\begin{proposition}[Threshold-resolution limit]
\label{prop:threshold-resolution}
Assume two policies \(\pi_\tau\) and \(\pi_{\tau'}\) induce identical replay outcomes outside the boundary region. Then any statistical procedure that distinguishes them with total error probability at most \(\delta\) must satisfy
\[
n\,m_{\tau,\tau'}
\gtrsim
\frac{B^2\log(1/\delta)}{\varepsilon^2}.
\]
where
$
\varepsilon
=
|\Delta_{\pi_\tau}-\Delta_{\pi_{\tau'}}|
$
and \(B\) represents a uniform bound on the normalized per-auction replay contribution. Equivalently, nearby threshold policies are statistically indistinguishable unless the effective boundary sample size \(n\,m_{\tau,\tau'}\) is sufficiently large.
\end{proposition}

The proposition gives an information-theoretic limit on replay resolution. Threshold policies can become impossible to distinguish not because the replay table is noisy globally, but because the logged marketplace contains too little mass in the boundary region where the two policies differ. The relevant notion of support is therefore local rather than global. Here, $B$ controls the intrinsic scale of replay variability near the threshold boundary. Larger values of \(B\) correspond to policies whose local replay outcomes can fluctuate more dramatically from one auction to another, which in turn requires more effective boundary support in order to reliably distinguish nearby policies. This observation changes how replay diagnostics should be interpreted. Large auction logs can still be locally uninformative around reserve thresholds, especially for aggressive floor increases or finely spaced threshold grids. Small replay differences between nearby policies should therefore not automatically be treated as meaningful rankings. Instead, threshold diagnostics should explicitly measure local boundary support and identify which policy distinctions are actually resolvable from the logged data.

Operationally, the proposition suggests that threshold policies separated by regions of weak support should be merged, coarsened, or deferred to online validation rather than ranked by small replay differences. The result also explains why replay frontiers can appear stable at coarse scales while becoming statistically fragile under fine-grained threshold perturbations.
\subsection{Lower-bound ranking and elimination}

Replay tables are useful for screening, but point-estimate rankings alone do not determine which policies are sufficiently supported to justify validation. In finite catalogs, weakly supported policies can appear attractive because of variance, sparse threshold support, or multiple comparisons. The relevant operational question is therefore not which policy has the largest estimated lift, but which policies remain competitive after simultaneous uncertainty is taken into account.

Let \(L_\alpha(\pi)\) and \(U_\alpha(\pi)\) denote simultaneous lower and upper confidence bounds for the replay lift \(\Delta_\pi\) over the gate-passing catalog \(\Ghat\subseteq\PiC\). The index \(\alpha\) records the allowed familywise error probability. Thus the event
\[
L_\alpha(\pi)\le \Delta_\pi \le U_\alpha(\pi)
\qquad
\forall \pi\in\Ghat
\]
is assumed to hold with probability at least \(1-\alpha\). We
define the lower-bound leader by
$
\pi^{LB}\in\arg\max_{\pi\in\Ghat}L_\alpha(\pi),
$
and define
\[
\Shat_\alpha(\rho)
=
\left\{
\pi\in\Ghat:
U_\alpha(\pi)\ge L_\alpha(\pi^{LB})-\rho
\right\}
\]
as the regret-tolerant shortlist.
Then the shortlist construction is conservative in the following sense. If a policy \(\pi\notin\Shat_\alpha(\rho)\), then by definition
$
U_\alpha(\pi)<L_\alpha(\pi^{LB})-\rho.
$
Coverage therefore implies
$
\Delta_\pi
\le
U_\alpha(\pi)
<
L_\alpha(\pi^{LB})-\rho
\le
\Delta_{\pi^{LB}}-\rho,
$
so
$
\Delta_{\pi^{LB}}-\Delta_\pi>\rho.
$
Any eliminated policy therefore has certified regret exceeding the tolerance level \(\rho\).

The same argument shows that the optimal gate-passing policy cannot be removed. Let
$
\pi^\star\in\arg\max_{\pi\in\Ghat}\Delta_\pi.
$
If \(\pi^\star\notin\Shat_\alpha(0)\), then
$U_\alpha(\pi^\star)<L_\alpha(\pi^{LB}),
$
which implies
$
\Delta_{\pi^\star}
<
\Delta_{\pi^{LB}},
$
contradicting the optimality of \(\pi^\star\). Hence
$
\pi^\star\in\Shat_\alpha(0).
$

Finally, if
$L_\alpha(\pi^{LB})
>
\max_{\pi\neq\pi^{LB}}U_\alpha(\pi),
$
then every competing policy satisfies
$
\Delta_\pi
\le
U_\alpha(\pi)
<
L_\alpha(\pi^{LB})
\le
\Delta_{\pi^{LB}},
$
so \(\pi^{LB}\) is uniquely optimal within the gate-passing catalog.

The construction converts simultaneous confidence intervals into a conservative decision rule. \textit{Policies are removed only when their optimistic value falls below the pessimistic value of the lower-bound leader by more than the tolerated regret threshold. The retained set therefore contains every policy that could still plausibly be optimal under simultaneous uncertainty, while eliminated policies have certified regret relative to the retained leader.}

\noindent \textbf{Why point-estimate ranking can fail.}
The decision rule matters as much as the estimator. Replay-only or OPE-only workflows typically rank policies by their estimated means. The support-aware rule instead ranks policies by simultaneous lower confidence bounds after support and safety gates. The distinction is important because noisy or weakly supported policies can achieve the largest point estimate while remaining statistically indefensible.

Consider two candidate policies \(a\) and \(b\) with true values
$
\Delta_a=1,
$ $
\Delta_b=0.8,
$
but estimated lifts
$
\widehat\Delta_a=1,
$ $
\widehat\Delta_b=1.2.
$
Suppose the simultaneous confidence intervals are
$
[L(a),U(a)]=[0.9,1.1],
$ $
[L(b),U(b)]=[0,2.4].
$
A point-estimate rule selects policy \(b\) because \(1.2>1\), even though \(b\) is truly inferior. The lower-bound rule instead selects \(a\), because its pessimistic value remains substantially positive while \(b\)'s lower bound collapses to zero. The example illustrates the practical role of support-aware ranking. The framework changes the burden of proof: a policy advances not because its estimated gain is largest, but because its lower-bound performance remains credible after accounting for uncertainty, multiplicity, and weak support.

\subsection{Segment-level non-harm}

Aggregate replay lift does not guarantee marketplace safety. A reserve policy can improve overall revenue while harming particular advertiser groups, exchanges, regions, or inventory classes. The segment-safety question is therefore whether the observed subgroup evidence is sufficient to rule out meaningful harm over the relevant segment space. The following proposition gives a sufficient sample-size condition for certifying global non-harm over a covered segment space.

\begin{proposition}[Sample size for global non-harm certification]
\label{prop:segment-sample-complexity}
Suppose segment rewards are bounded in \([0,A]\), the analyst evaluates \(K\) observed grid segments \(s_0\in\mathcal S_0\), and each lower confidence bound \(L_\alpha(\pi,s_0)\) is formed using a union-bound confidence radius. Assume further that the segment lift function \(\Delta_\pi(s)\) is \(L_s\)-Lipschitz over the segment space \(\mathcal S\), and that the observed grid \(\mathcal S_0\) forms a \(\rho\)-cover of \(\mathcal S\). To certify
$
\Delta_\pi(s)\ge 0
\qquad
\forall s\in\mathcal S,
$
it suffices that every observed segment satisfies
\[
n_{s_0}
\gtrsim
\frac{A^2\log(K/\alpha)}
{(\eta+L_s\rho)^2},
\]
where \(n_{s_0}\) is the sample size of segment \(s_0\), and
$
\eta
:=
\min_{s_0\in\mathcal S_0}
L_\alpha(\pi,s_0)
$
is the minimum lower-bound margin over the observed grid. Equivalently, global non-harm is certified once every displayed segment has lower confidence bound at least \(L_s\rho\) above zero.
\end{proposition}

The proposition exposes the operational cost of subgroup safety. The \(\log(K/\alpha)\) term is the multiplicity penalty for simultaneously checking many segments, while the \((\eta+L_s\rho)^{-2}\) term quantifies the difficulty of certifying small non-harm margins under imperfect segment coverage. The Lipschitz-cover term \(L_s\rho\) represents the uncertainty introduced by extending guarantees from the observed segment grid \(\mathcal S_0\) to nearby unobserved segments in the full segment space \(\mathcal S\). Consequently, sparse or weakly supported segments should be treated as unresolved rather than implicitly safe, especially when the observed subgroup margins are small or the segment cover is coarse.

\subsection{Unified support-aware decision guarantee}

The preceding results characterize the four stages of the proposed offline decision pipeline:
(i) localized replay concentration and bounded-response ranking for the threshold-supported replay component (Theorem~\ref{thm:replay-generalization} and Corollary~\ref{cor:support-response}),
(ii) threshold-resolution limits (Proposition~\ref{prop:threshold-resolution}),
(iii) conservative shortlist construction via the lower-bound elimination rule in Section~4.3, and
(iv) segment-level non-harm certification (Proposition~\ref{prop:segment-sample-complexity} and Lemma~\ref{thm:segment-nonharm}).
Taken together, these results imply a unified guarantee for the full support-aware policy-selection procedure.

Let
$
\mathcal E_{\mathrm{rep}}
$
denote the event on which Theorem~\ref{thm:replay-generalization} and Corollary~\ref{cor:support-response} hold, let
$
\mathcal E_{\mathrm{thr}}
$
denote the event on which Proposition~\ref{prop:threshold-resolution} holds, and let
$
\mathcal E_{\mathrm{seg}}
$
denote the event on which Proposition~\ref{prop:segment-sample-complexity} and Lemma~\ref{thm:segment-nonharm} hold. Let
\[
\mathcal E_{\mathrm{lb}}
:=
\left\{
L_\alpha(\pi)\le \Delta_\pi \le U_\alpha(\pi)
\ \text{for all gate-passing policies }\pi\in\Ghat
\right\}
\]
be the simultaneous lower-upper coverage event used by the lower-bound elimination rule in Section~4.3. Then on the joint event
$
\mathcal E
:=
\mathcal E_{\mathrm{rep}}\cap \mathcal E_{\mathrm{thr}}\cap \mathcal E_{\mathrm{lb}}\cap \mathcal E_{\mathrm{seg}},
$ we have the following global guarantee.
\begin{theorem}[Unified support-aware decision guarantee]
\label{thm:reserve-pipeline}
On the event \(\mathcal E\), the support-aware pipeline outputs a partition
$
\Ghat
=
\mathcal C_\alpha
\cup
\mathcal D_\alpha
\cup
\mathcal U_\alpha,
$
where the three sets are disjoint and correspond to certified, dominated, and unresolved policies such that:
\begin{enumerate}
\item \textbf{Certification.} If \(\pi\in\mathcal C_\alpha\), then
$
\Delta_\pi(s)\ge 0 $, $ \forall s\in\mathcal S,
$
so \(\pi\) is certified non-harmful over the covered segment space.

\item \textbf{Dominance elimination.} If \(\pi\in\mathcal D_\alpha\), then
$
\Delta_{\pi^{LB}}-\Delta_\pi > \rho,
$
where \(\pi^{LB}\in\arg\max_{\pi\in\Ghat} L_\alpha(\pi)\) is the lower-bound leader and \(\rho\) is the shortlist tolerance. In particular, every eliminated policy has certified regret exceeding \(\rho\).

\item \textbf{Best-supported policy retention.} The best gate-passing policy remains in the retained shortlist:
$
\arg\max_{\pi\in\Ghat}\Delta_\pi \in \mathcal C_\alpha \cup \mathcal U_\alpha,
$ and, whenever the lower-bound leader is uniquely separated from the rest of the catalog by the simultaneous bounds, it is retained as the validation target.

\item \textbf{Unresolved policies.} If \(\pi\in\mathcal U_\alpha\), then \(\pi\) fails at least one of the following: (i) boundary support, (ii) subgroup support, (iii) replay margin, or (iv) simultaneous separation. Equivalently, unresolved policies are precisely those for which the available evidence is insufficient for either certification or elimination.
\end{enumerate}
\end{theorem}

The theorem is a compositional statement rather than a new concentration bound. The localized replay result controls the threshold-supported component of fixed-bid estimation, Proposition~\ref{prop:threshold-resolution} limits what nearby threshold rules can be distinguished from the logs, the lower-bound elimination rule turns simultaneous uncertainty into a conservative shortlist, and Proposition~\ref{prop:segment-sample-complexity} together with Lemma~\ref{thm:segment-nonharm} prevents aggregate lift from masking segment-level harm. The lower-bound event \(\mathcal E_{\mathrm{lb}}\) is the part of the pipeline that covers the full empirical replay summaries used for practical shortlist construction. The resulting decision object therefore separates certified validation targets, statistically dominated alternatives, and unresolved candidates requiring additional evidence or online experimentation. All conclusions remain conditional on the fixed-bid replay estimand; bidder adaptation and marketplace interference remain outside the offline guarantee.

\section{Experiments}
\label{sec:experiments}

The experiments evaluate the support-aware decision pipeline on public iPinYou real-time-bidding logs \cite{liao2014ipinyou}. Season two is the offline development panel and season three is held out for frozen out-of-time replay validation. The season-two panel contains \(53{,}289{,}330\) auction opportunities, while the season-three panel contains \(10{,}566{,}743\) opportunities. The chosen catalog contains 19 reserve policies, including the logged baseline, uniform floor increases, empirical quantile floors, and margin-gated rules. Appendix~\ref{app:implementation} reports the full catalog so that the policy definitions can be checked directly.

The empirical question we target here is whether the logged evidence can be converted into a conservative decision object with three parts. A policy may be certified as a validation target, eliminated as statistically dominated, or retained as unresolved because the data do not support a sharper conclusion. The experiments follow the theoretical pipeline in three steps. Section~\ref{subsec:exp-shortlist} constructs the conservative shortlist and tests the finite-catalog replay and elimination logic. Section~\ref{subsec:exp-support} studies threshold-resolution support through boundary-window diagnostics. Section~\ref{subsec:exp-validation} checks out-of-time transfer and subgroup safety. Appendix~\ref{app:additional-experiments} reports additional robustness checks, including catalog-size sensitivity, pairwise boundary-support diagnostics, shortlist stability, and expanded segment-safety analyses.

\subsection{Conservative shortlist construction}
\label{subsec:exp-shortlist}

The first experiment applies the finite-catalog lower-bound elimination step in Theorem~\ref{thm:reserve-pipeline}. Each policy is replayed on season two. We then compute the aggregate replay lift, the daily replay lift used for simultaneous uncertainty control, lower and upper confidence bounds, and the final decision label in
$
\{\text{certified},\ \text{dominated},\ \text{unresolved}\}.
$ Specifically, for each policy, the confidence bounds are computed from daily replay-lift variation using a Bonferroni-adjusted normal interval over the finite catalog,
\[
\mathrm{LCB}_\pi
=
\widehat{\Delta}_\pi-z_{1-\alpha/(2|\PiC|)}\widehat{\mathrm{se}}_\pi,
\qquad
\mathrm{UCB}_\pi
=
\widehat{\Delta}_\pi+z_{1-\alpha/(2|\PiC|)}\widehat{\mathrm{se}}_\pi,
\]
where \(\widehat{\mathrm{se}}_\pi\) is the standard error of the policy's daily replay lift.
This experiment uses full replay summaries and the simultaneous lower-bound event \(\mathcal E_{\mathrm{lb}}\). The localized replay result in Theorem~\ref{thm:replay-generalization} is tested more directly in Section~\ref{subsec:exp-support}, where boundary-window diagnostics estimate how much local threshold support is available.
\begin{figure}[t]
\centering
\includegraphics[width=0.98\linewidth]{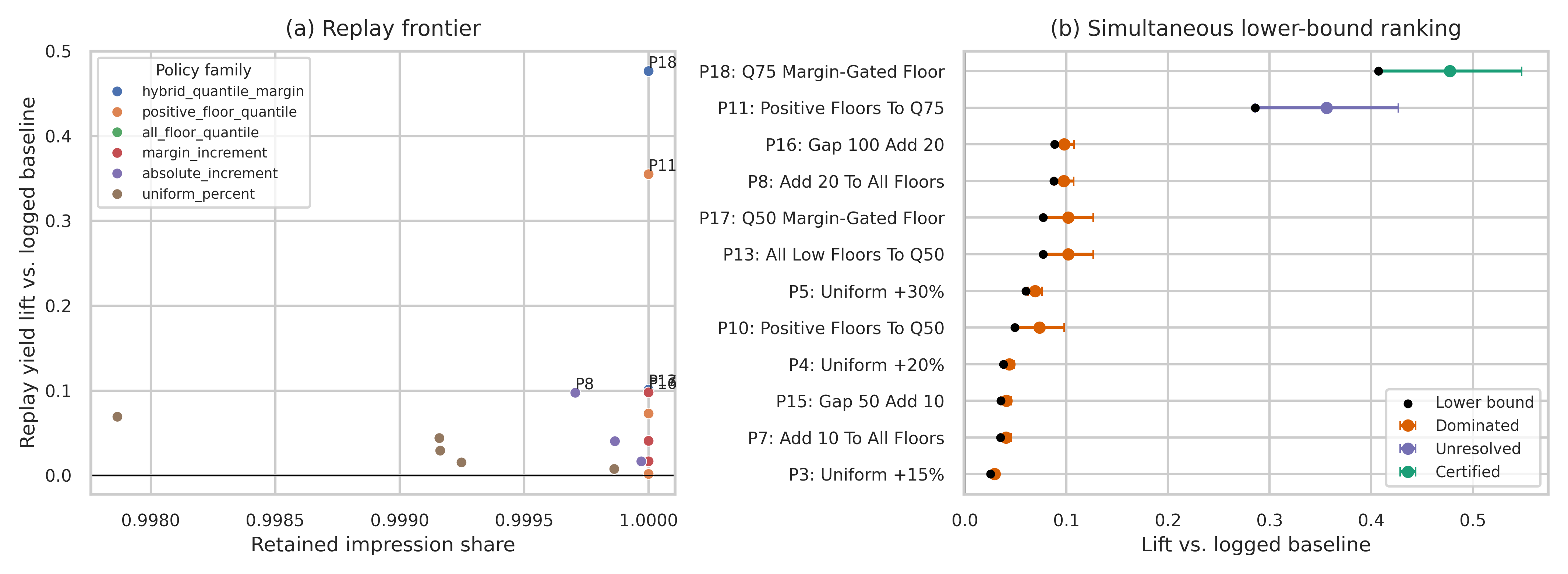}
\caption{\small Conservative shortlist construction on season two. Panel (a) shows the replay frontier for non-baseline reserve policies, with replay yield lift plotted against retained impression share. Panel (b) shows simultaneous lower-bound ranking. Colored points indicate whether the decision rule certifies the policy, eliminates it as dominated, or leaves it unresolved. Black points show support-adjusted lower bounds.}
\label{fig:exp-shortlist}
\end{figure}

Figure~\ref{fig:exp-shortlist} shows the first main empirical result. The replay frontier identifies P18, the Q75 margin-gated floor rule, as the point-estimate leader with a \(47.66\%\) aggregate season-two replay lift. The conservative ranking subtracts both the simultaneous uncertainty radius and a retained-impression support penalty,
$
\mathrm{LCB}^{\mathrm{support}}_\pi
=
\mathrm{LCB}_\pi
-
\lambda\bigl(1-\widehat r_\pi\bigr),
$
where \(\widehat r_\pi\) is the retained-impression share and \(\lambda=1\) in the experiments. Under this rule, P18 remains the lower-bound leader with a \(40.71\%\) certified lower-bound lift. P11, the positive-floors-to-Q75 rule, remains unresolved with a \(28.57\%\) lower-bound lift. The remaining 17 policies are eliminated because their upper bounds (\(\mathrm{UCB}_\pi\)) fall below the support-adjusted lower bound (\(\mathrm{LCB}^{\mathrm{support}}_\pi\)) of P18. The result illustrates why the paper frames offline learning as decision support rather than winner selection. Replay-only ranking would simply select P18, as seen in Fig.~\ref{fig:exp-shortlist}(a). The support-aware decision object is more informative. It certifies one validation target, retains one unresolved competitor, and removes 17 dominated alternatives, as shown in Fig.~\ref{fig:exp-shortlist}(b). Appendix~\ref{app:elimination} shows that this two-policy shortlist is stable over elimination tolerances from \(0\) to \(0.10\), and Appendix~\ref{app:replay-diagnostics} shows that bootstrap replay resampling selects P18 in all \(1000\) bootstrap draws.

\subsection{Support-localized threshold resolution}
\label{subsec:exp-support}

The second experiment targets the threshold-resolution result in Proposition~\ref{prop:threshold-resolution} and gives an empirical diagnostic for the \(q\)-local radius in Theorem~\ref{thm:replay-generalization}. Replay lift is always computed on the full season-two panel. We then vary a diagnostic boundary-window width \(h\), count observations satisfying \(|b_i-f_i^\pi|\le h\), and apply an inverse-square-root support penalty based on this boundary count. For any fixed evidence fraction \(q\), the smallest \(h\) whose boundary count reaches \(q n\) is the empirical analogue of \(r_\pi(q)\). Thus the sweep shows how quickly local evidence accumulates around policy thresholds, while the policy effect itself is not recomputed using only the observations inside the window.

\begin{figure}[t]
\centering
\includegraphics[width=0.98\linewidth]{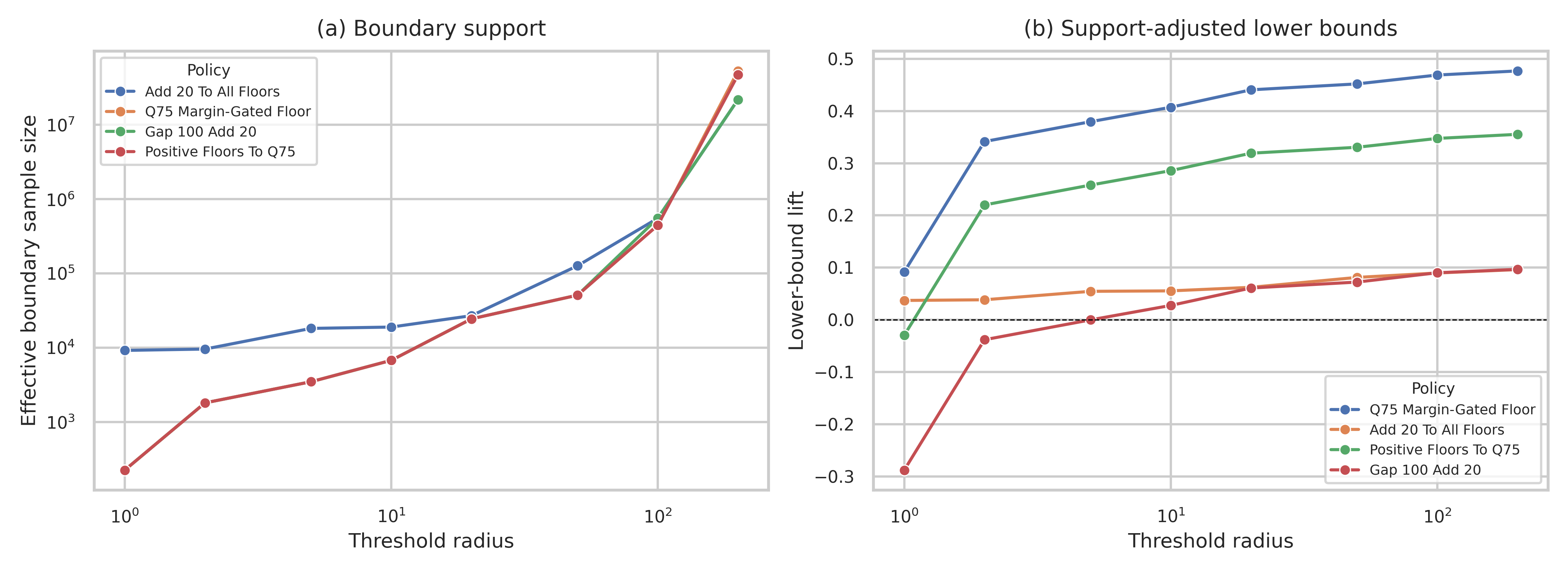}
\caption{\small Support-localized threshold resolution. Panel (a) reports effective boundary sample size \(n_{\mathrm{boundary}}(h)\) as the diagnostic boundary window expands. Panel (b) reports support-adjusted lower-bound lift for the leading policies. The same season-two panel can be statistically large overall while remaining locally thin near narrow reserve-threshold bands.}
\label{fig:exp-support-resolution}
\end{figure}

Figure~\ref{fig:exp-support-resolution} confirms the local-support limitation. At boundary-window width \(h=1\), the median boundary sample size is only 225 observations and only two policies certify positive support-adjusted lift. At \(h=10\), the median boundary sample size rises to 6,757 and ten policies certify positive lower-bound lift. At \(h=200\), the median boundary sample size reaches \(20.82\) million and 18 policies certify positive lower-bound lift. The maximum support-adjusted lower bound rises from \(9.17\%\) at \(h=1\) to \(47.67\%\) at \(h=200\), not because the policy effect changes, but because the diagnostic support penalty becomes smaller as the boundary window contains more data.

These results are the empirical counterpart of the threshold-resolution limit. The full season-two panel has more than 53 million opportunities, yet the informative sample for narrow threshold comparisons can be only a few hundred observations. The top-\(k\) stability analysis reinforces this point. P18 is separated from P11 by a lower-bound margin of \(12.15\) percentage points, but lower ranks contain near ties. For example, the margin between the third and fourth lower-bound ranks is only \(0.069\) percentage points, and one top-five boundary is tied exactly. The practical implication is that the framework can certify a robust leader and shortlist while refusing to overinterpret the entire ranking. Appendix~\ref{app:support-diagnostics} extends this analysis with a finer boundary-window grid and pairwise boundary-support distributions.

A complementary localization check, reported in Appendix~\ref{app:q-localized}, ranks policies using only observations closest to each policy's reserve threshold. This diagnostic does not replace aggregate replay, because it normalizes by local boundary yield rather than total marketplace yield. It instead asks whether a policy's apparent value is supported by strong local evidence near the threshold where the policy changes auction outcomes. The check identifies P11, Positive Floors To Q75, as the \(q\)-localized boundary-lift leader for all tested localization levels \(q\in\{0.01,0.025,0.05,0.10,0.20\}\), while P18 remains the full-replay leader. This distinction is useful. P11 is locally efficient near its threshold and therefore remains a serious unresolved competitor, whereas P18 has the stronger aggregate replay and out-of-time transfer profile.

% \[
% \text{fine threshold resolution}
% \quad\Longleftrightarrow\quad
% \text{strong local support}.
% \]

\subsection{Validation readiness through out-of-time transfer and segment safety}
\label{subsec:exp-validation}

The final main experiment asks whether the season-two shortlist remains plausible in a new logged environment and whether the priority policy hides subgroup harm. The policy catalog and season-two quantiles are frozen, then replayed on season three without refitting. In parallel, P18 is evaluated across advertiser, exchange, and region segments using segment-level confidence bars. This step relates the bounded-response caution in Corollary~\ref{cor:support-response} to out-of-time replay stability, and it directly evaluates the subgroup-safety components in Proposition~\ref{prop:segment-sample-complexity}, Lemma~\ref{thm:segment-nonharm}, and the unified decision guarantee in Theorem~\ref{thm:reserve-pipeline}.

\begin{figure}[t]
\centering
\includegraphics[width=0.98\linewidth]{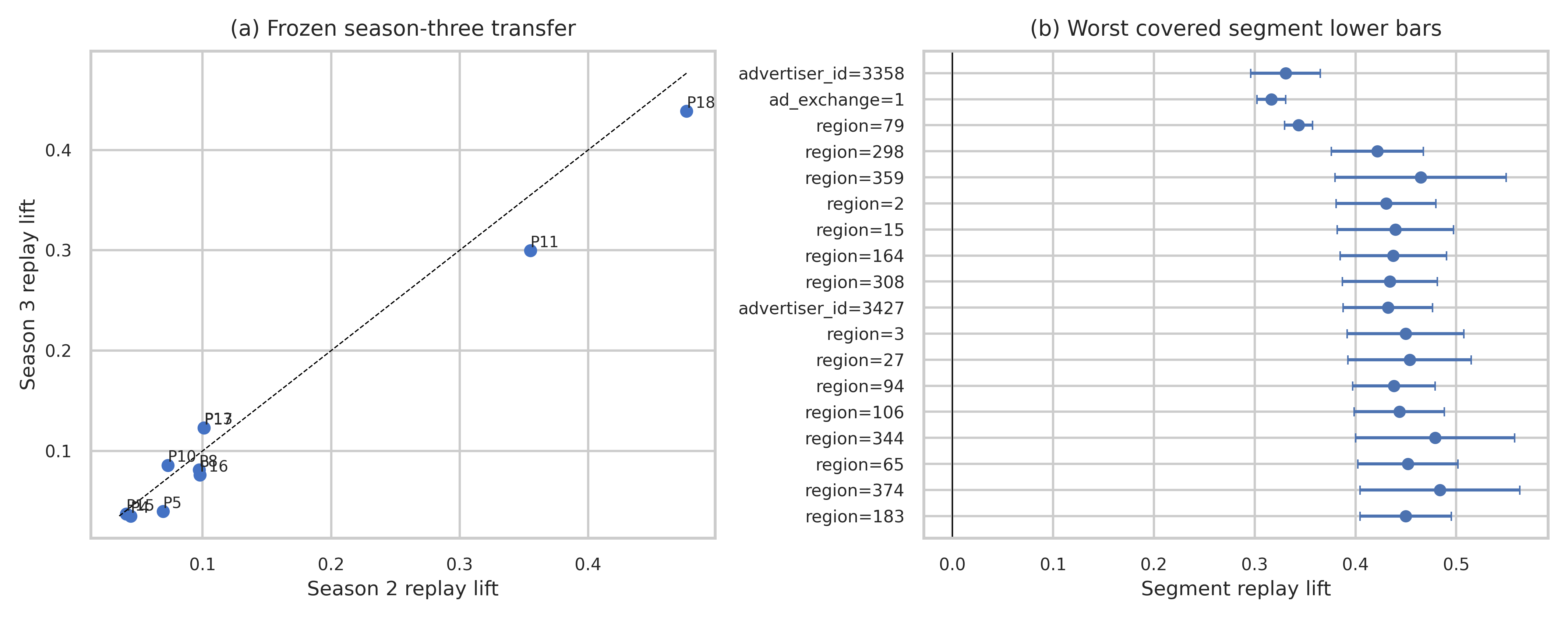}
\caption{\small Validation readiness through transfer and subgroup safety. Panel (a) compares season-two and season-three replay lifts under the frozen catalog. Panel (b) reports mean segment-level replay lifts for the covered segments with the smallest lower endpoints, with 95\% normal confidence bars computed from daily segment-level replay variation.}
\label{fig:exp-validation-readiness}
\end{figure}

Figure~\ref{fig:exp-validation-readiness} shows that the priority policy transfers well under frozen replay. P18 remains rank 1 in season three with a \(43.87\%\) replay lift and \(100.00\%\) retained-impression share. P11 remains rank 2 with a \(29.97\%\) season-three replay lift. Across the full catalog, the season-two to season-three Spearman rank correlation is \(0.988\), and four of the season-two top five policies remain in the season-three top five, as reported in Appendix~\ref{app:transfer}.

The segment-safety panel shows no covered subgroup harm under the evaluated grid. The main segment check covers 44 advertiser, exchange, and region segments. All 44 have nonnegative lower confidence endpoints, and the smallest lower endpoint is \(29.61\%\). This does not prove live deployment safety, because bidders may adapt after observing new reserves. It does show that the priority policy survives the offline gates available in the logged data. The correct output is therefore validation readiness. The policy is strong enough to justify online validation, but the fixed-bid replay contract still leaves bidder response, pacing changes, and marketplace interference to be tested online. Appendix~\ref{app:segment-diagnostics} expands the segment grid to 58 segments, adding inventory-category and bid-gap strata, and again finds no negative lower endpoint.

\section{Conclusion and future work}
\label{sec:conclusion}

This paper argues that offline reserve-price analysis should produce certified decisions rather than point-estimate rankings alone. The proposed framework combines replay evaluation, simultaneous lower-bound ranking, support-localized threshold diagnostics, policy elimination, and segment non-harm checks into a single finite-catalog decision object. On iPinYou logs, this object certifies one leading validation target, removes statistically dominated alternatives, and preserves an unresolved competitor that remains relevant for online testing. This is the intended role of the method. It narrows the validation agenda without converting static replay evidence into a deployment claim.

Several extensions are natural. The first is to connect the offline decision object to live experimentation designs that explicitly handle bidder response, pacing, and marketplace interference. The second is to extend the finite-catalog guarantee to structured policy classes while retaining the operational interpretability of catalog-based review. The third is to strengthen segment safety by replacing fixed subgroup grids with adaptive or representation-based segment discovery. Finally, future work should study how support-aware offline screening and online validation can be combined into an iterative marketplace policy-learning loop.

\paragraph{Code availability.}
The paper repository is available at \href{https://github.com/p-shekhar/offline-policy-learning.git}{p-shekhar/offline-policy-learning.git}.

\bibliographystyle{plainnat}
\bibliography{references}

\newpage
\appendix
\section{Proofs}

\subsection{Proof of Theorem~\ref{thm:replay-generalization}}

\begin{proof}
For each \(\pi \in \PiC\), define the centered replay difference
\[
Z_i^\pi := Y_i^\pi - Y_i^0,
\qquad
\mu_{Z,\pi} := \mathbb{E}[Z_i^\pi].
\]
Then
\[
\Delta_\pi = \frac{\mu_{Z,\pi}}{\mu_0}.
\]

Fix \(\pi\) and \(q\in(0,1)\). Let
\[
A_{\pi,q}:=\{|G_i-\tau_\pi|\le r_\pi(q)\},
\qquad
m_{\pi,q}=\mathbb{P}(A_{\pi,q}).
\]
By assumption, the replay difference admits the decomposition
\[
Z_i^\pi=Z_{i,q}^\pi+R_{i,q}^\pi,
\]
where \(Z_{i,q}^\pi\) is supported on \(A_{\pi,q}\), \(|Z_{i,q}^\pi|\le B\), and
\[
\bigl|\mathbb{E}[R_{i,q}^\pi]\bigr|\le L_\pi r_\pi(q).
\]
The localized estimator is
\[
\widehat{\Delta}_{\pi,q}
=
\frac{1}{n\mu_0}\sum_{i=1}^n Z_{i,q}^\pi.
\]
It remains to control the difference between the sample mean of \(Z_{i,q}^\pi\) and the full mean \(\mu_{Z,\pi}\), and then divide by \(\mu_0\).

We first control the local stochastic term. Since \(Z_{i,q}^\pi\) is supported on an event with probability \(m_{\pi,q}\), we have
\[
\mathbb{E}\!\left[(Z_{i,q}^\pi)^2\right]\le B^2m_{\pi,q},
\qquad
|Z_{i,q}^\pi|\le B.
\]
Bernstein's inequality therefore implies that for a fixed \(\pi\),
\[
\Pr\!\left(
\left|\frac{1}{n}\sum_{i=1}^n Z_{i,q}^\pi - \mathbb{E}[Z_{i,q}^\pi]\right|
>
C_1 B\sqrt{\frac{m_{\pi,q}\log(2/\eta)}{n}}
+
C_2 B\frac{\log(2/\eta)}{n}
\right)
\le \eta,
\]
for universal constants \(C_1,C_2>0\). Applying a union bound over the finite catalog \(\PiC\) with \(\eta=\delta/|\PiC|\) gives, with probability at least \(1-\delta\),
\[
\sup_{\pi\in\PiC}
\left|
\frac{1}{n}\sum_{i=1}^n Z_{i,q}^\pi - \mathbb{E}[Z_{i,q}^\pi]
\right|
\le
\sup_{\pi\in\PiC}
\left[
C\!B\sqrt{\frac{m_{\pi,q}\log(2|\PiC|/\delta)}{n}}
+
C\!B\frac{\log(2|\PiC|/\delta)}{n}
\right]
\]
for a universal constant \(C>0\).

Now decompose the localized replay error:
\[
\frac{1}{n}\sum_{i=1}^n Z_{i,q}^\pi-\mu_{Z,\pi}
=
\left(\frac{1}{n}\sum_{i=1}^n Z_{i,q}^\pi-\mathbb{E}[Z_{i,q}^\pi]\right)
-
\mathbb{E}[R_{i,q}^\pi].
\]
The first term is controlled by the concentration bound above. The second term is the mean residual approximation error, which is at most \(L_\pi r_\pi(q)\) by assumption.
Combining the two displays and dividing by \(\mu_0\) yields, with probability at least \(1-\delta\),
\[
\left|\widehat{\Delta}_{\pi,q}-\Delta_\pi\right|
\le
C\frac{B}{\mu_0}\sqrt{\frac{m_{\pi,q}\log(2|\PiC|/\delta)}{n}}
+
C\frac{B}{\mu_0}\frac{\log(2|\PiC|/\delta)}{n}
+
\frac{L_\pi r_\pi(q)}{\mu_0}.
\]
Taking the supremum over \(\pi\in\PiC\) yields the theorem.
\end{proof}

\subsection{Corollary~\ref{cor:replay-regret}}

\begin{corollary}[Replay regret of the localized empirical maximizer]
\label{cor:replay-regret}
Let
\[
\widehat{\pi}_q\in\arg\max_{\pi\in\PiC}\widehat{\Delta}_{\pi,q},
\qquad
\pi^\star\in\arg\max_{\pi\in\PiC}\Delta_\pi.
\]
On the event of Theorem~\ref{thm:replay-generalization},
\[
\Delta_{\pi^\star}-\Delta_{\widehat{\pi}_q}
\le
2\sup_{\pi\in\PiC}
\left[
C\frac{B}{\mu_0}\sqrt{\frac{m_{\pi,q}\log(2|\PiC|/\delta)}{n}}
+
C\frac{B}{\mu_0}\frac{\log(2|\PiC|/\delta)}{n}
+
\frac{L_\pi r_\pi(q)}{\mu_0}
\right].
\]
In particular, the replay regret of the localized empirical maximizer is governed by the local boundary mass and the \(q\)-local radius in the replay-localization bound, not by catalog size alone.
\end{corollary}

This corollary is the regret version of Theorem~\ref{thm:replay-generalization}. It should be read as a statement about the policy selected by the \(q\)-localized score \(\widehat{\Delta}_{\pi,q}\). It does not claim that the full empirical replay maximizer has the same regret unless the empirical residual component is also controlled. The point is narrower and useful for diagnostics: when a practitioner ranks policies using the locally supported part of the replay contrast, the selected policy's population replay regret is controlled by the same boundary-mass and local-radius terms that appear in the theorem.

\begin{proof}
Let
\[
\widehat{\pi}_q\in\arg\max_{\pi\in\PiC}\widehat{\Delta}_{\pi,q},
\qquad
\pi^\star\in\arg\max_{\pi\in\PiC}\Delta_\pi.
\]
Then
\[
\Delta_{\pi^\star}-\Delta_{\widehat{\pi}_q}
=
\bigl(\Delta_{\pi^\star}-\widehat{\Delta}_{\pi^\star,q}\bigr)
+
\bigl(\widehat{\Delta}_{\pi^\star,q}-\widehat{\Delta}_{\widehat{\pi}_q,q}\bigr)
+
\bigl(\widehat{\Delta}_{\widehat{\pi}_q,q}-\Delta_{\widehat{\pi}_q}\bigr).
\]
Because \(\widehat{\pi}_q\) maximizes \(\widehat{\Delta}_{\pi,q}\),
\[
\widehat{\Delta}_{\pi^\star,q}-\widehat{\Delta}_{\widehat{\pi}_q,q} \le 0.
\]
Hence
\[
\Delta_{\pi^\star}-\Delta_{\widehat{\pi}_q}
\le
\left|\Delta_{\pi^\star}-\widehat{\Delta}_{\pi^\star,q}\right|
+
\left|\widehat{\Delta}_{\widehat{\pi}_q,q}-\Delta_{\widehat{\pi}_q}\right|
\le
2\sup_{\pi\in\PiC}\left|\widehat{\Delta}_{\pi,q}-\Delta_\pi\right|.
\]
Applying Theorem~\ref{thm:replay-generalization} gives the stated bound.
\end{proof}

\subsection{Proof of Corollary~\ref{cor:support-response}}
\begin{proof}
On the event of Theorem~\ref{thm:replay-generalization},
\[
\Delta_\pi \ge \widehat{\Delta}_{\pi,q}-\epsilon_\pi(q),
\qquad
\Delta_{\pi'} \le \widehat{\Delta}_{\pi',q}+\epsilon_{\pi'}(q).
\]
Therefore,
\[
\Delta_\pi-\Delta_{\pi'}
\ge
\bigl(\widehat{\Delta}_{\pi,q}-\widehat{\Delta}_{\pi',q}\bigr)
-
(\epsilon_\pi(q)+\epsilon_{\pi'}(q)).
\]
Since \(V_\pi = R_\pi+\Gamma_\pi\), the live difference is
\[
V_\pi-V_{\pi'}
=
(R_\pi-R_{\pi'})+(\Gamma_\pi-\Gamma_{\pi'}).
\]
After normalization by \(\mu_0\), the replay gap contributes \(\Delta_\pi-\Delta_{\pi'}\), so if
\[
\widehat{\Delta}_{\pi,q}-\widehat{\Delta}_{\pi',q}
>
\epsilon_\pi(q)+\epsilon_{\pi'}(q)+\frac{\Gamma_{\pi'}-\Gamma_\pi}{\mu_0},
\]
then the lower bound on the replay gap dominates the response gap, implying \(V_\pi>V_{\pi'}\).
If instead \(|\Gamma_\pi-\Gamma_{\pi'}|\le \eta\), the conclusion follows from
\[
\Delta_\pi-\Delta_{\pi'}
\ge
\widehat{\Delta}_{\pi,q}-\widehat{\Delta}_{\pi',q}
-
(\epsilon_\pi(q)+\epsilon_{\pi'}(q))
>
\frac{\eta}{\mu_0}.
\]
\end{proof}

\subsection{Proof of Proposition~\ref{prop:threshold-resolution}}
Before proving the proposition, we first state the underlying information-theoretic lemma.

\begin{lemma}[Information-theoretic barrier for threshold resolution]
\label{lem:threshold-resolution}
Let \(\pi_\tau\) and \(\pi_{\tau'}\) be two threshold policies with separation \(d_{\tau,\tau'}=|\tau-\tau'|>0\). Define the separating boundary region
\[
A_{\tau,\tau'}
=
\left\{
\min(\tau,\tau')
\le G_i \le
\max(\tau,\tau')
\right\},
\qquad
m_{\tau,\tau'}=\Pr(A_{\tau,\tau'}).
\]
Assume the two policies induce identical replay outcomes outside \(A_{\tau,\tau'}\). Suppose further that, conditional on \(A_{\tau,\tau'}\), the one-sample KL divergence between the induced replay laws satisfies
\[
\mathrm{KL}\!\left(P_\tau(\cdot \mid A_{\tau,\tau'}) \,\|\, P_{\tau'}(\cdot \mid A_{\tau,\tau'})\right)
\le c_0\,\frac{\varepsilon^2}{B^2},
\]
for some universal constant \(c_0>0\), where
\[
\varepsilon := |\Delta_{\pi_\tau}-\Delta_{\pi_{\tau'}}|
\]
is the normalized replay gap and \(B\) is an almost sure bound on the normalized per-auction reward difference.

Then the \(n\)-sample replay laws satisfy
\[
\mathrm{KL}\!\left(P_\tau^{(n)}\,\|\,P_{\tau'}^{(n)}\right)
\le
c_0\,n\,m_{\tau,\tau'}\,\frac{\varepsilon^2}{B^2}.
\]
Consequently, for any test \(\phi\) that tries to distinguish \(\pi_\tau\) from \(\pi_{\tau'}\), where \(\phi=1\) means deciding in favor of \(\pi_\tau\),
\[
\inf_{\phi}
\Bigl\{
P_{\tau}^{(n)}(\phi=0)+P_{\tau'}^{(n)}(\phi=1)
\Bigr\}
\ge
\frac12
\exp\!\left(
-c_0\,n\,m_{\tau,\tau'}\,\frac{\varepsilon^2}{B^2}
\right).
\]
\end{lemma}

\begin{proof}
Let \(P_\tau^{(n)}\) and \(P_{\tau'}^{(n)}\) denote the \(n\)-sample replay laws induced by \(\pi_\tau\) and \(\pi_{\tau'}\). By assumption, the two policies differ only on the separating boundary region
\[
A_{\tau,\tau'}
=
\left\{
\min(\tau,\tau')
\le G_i \le
\max(\tau,\tau')
\right\},
\qquad
m_{\tau,\tau'}=\Pr(A_{\tau,\tau'}).
\]
Outside \(A_{\tau,\tau'}\), the replay outcomes are identical, so the one-sample laws differ only on a set of probability mass \(m_{\tau,\tau'}\).

By the assumed local regularity condition, the conditional one-sample KL divergence inside the separating boundary region is at most
\[
\mathrm{KL}\!\left(P_\tau(\cdot \mid A_{\tau,\tau'}) \,\|\, P_{\tau'}(\cdot \mid A_{\tau,\tau'})\right)
\le c_0\,\frac{\varepsilon^2}{B^2}.
\]
Since the policies coincide outside \(A_{\tau,\tau'}\), the unconditional one-sample KL divergence is bounded by the boundary mass times the conditional KL:
\[
\mathrm{KL}(P_\tau \,\|\, P_{\tau'})
\le
m_{\tau,\tau'}\,
\mathrm{KL}\!\left(P_\tau(\cdot \mid A_{\tau,\tau'}) \,\|\, P_{\tau'}(\cdot \mid A_{\tau,\tau'})\right)
\le
c_0\,m_{\tau,\tau'}\,\frac{\varepsilon^2}{B^2}.
\]
Because the \(n\) logged opportunities are independent, KL tensorizes:
\[
\mathrm{KL}\!\left(P_\tau^{(n)} \,\|\, P_{\tau'}^{(n)}\right)
=
n\,\mathrm{KL}(P_\tau \,\|\, P_{\tau'})
\le
c_0\,n\,m_{\tau,\tau'}\,\frac{\varepsilon^2}{B^2}.
\]

Now let \(\phi\) be any test for distinguishing \(P_\tau^{(n)}\) from \(P_{\tau'}^{(n)}\), where \(\phi=1\) means deciding in favor of \(\pi_\tau\). The Bretagnolle--Huber inequality gives
\[
P_\tau^{(n)}(\phi=0)+P_{\tau'}^{(n)}(\phi=1)
\ge
\frac12 \exp\!\left(
-\mathrm{KL}\!\left(P_\tau^{(n)} \,\|\, P_{\tau'}^{(n)}\right)
\right).
\]
Substituting the KL bound yields
\[
P_\tau^{(n)}(\phi=0)+P_{\tau'}^{(n)}(\phi=1)
\ge
\frac12
\exp\!\left(
-c_0\,n\,m_{\tau,\tau'}\,\frac{\varepsilon^2}{B^2}
\right).
\]
Taking the infimum over all tests \(\phi\) proves the stated lower bound.
\end{proof}

\begin{proof}[Proof of Proposition~\ref{prop:threshold-resolution}]
If a procedure distinguishes \(\pi_\tau\) and \(\pi_{\tau'}\) with total error probability at most \(\delta\), then for some test \(\phi\),
\[
P_\tau^{(n)}(\phi=0)+P_{\tau'}^{(n)}(\phi=1)\le \delta.
\]
By Lemma~\ref{lem:threshold-resolution},
\[
\frac12
\exp\!\left(
-c_0\,n\,m_{\tau,\tau'}\,\frac{\varepsilon^2}{B^2}
\right)
\le
\delta.
\]
Rearranging gives
\[
n\,m_{\tau,\tau'}
\ge
\frac{B^2}{c_0\varepsilon^2}\log\!\left(\frac{1}{2\delta}\right).
\]
Up to a universal constant factor, this is
\[
n\,m_{\tau,\tau'}\gtrsim \frac{B^2\log(1/\delta)}{\varepsilon^2}.
\]
Thus the relevant quantity is the effective boundary sample \(n\,m_{\tau,\tau'}\), not the full panel size \(n\).
\end{proof}

\subsection{Proof of Proposition~\ref{prop:segment-sample-complexity}}

We first establish the following uniform segment-safety lemma. Proposition~\ref{prop:segment-sample-complexity} then follows directly with a union-bound concentration argument.

\begin{lemma}[Uniform segment-safety certificate]
\label{thm:segment-nonharm}
Let \(\Delta_\pi(s)\) denote the lift of policy \(\pi\) in segment \(s\), and let \(\mathcal S_0\) be a finite \(\rho\)-cover of the segment space \(\mathcal S\). Assume \(\Delta_\pi(\cdot)\) is \(L_s\)-Lipschitz, and that on an event of probability at least \(1-\alpha\),
\[
\Delta_\pi(s_0)\ge L_\alpha(\pi,s_0)
\qquad\text{for all } s_0\in\mathcal S_0.
\]
If
\[
\min_{s_0\in\mathcal S_0} L_\alpha(\pi,s_0)\ge \eta,
\]
then on the same event,
\[
\Delta_\pi(s)\ge \eta-L_s\rho
\qquad\text{for all } s\in\mathcal S.
\]
In particular, if \(\eta>L_s\rho\), then \(\pi\) is certified non-harmful over \(\mathcal S\); if \(\eta\ge L_s\rho+\gamma\), then \(\Delta_\pi(s)\ge \gamma\) uniformly over \(\mathcal S\).
\end{lemma}

\begin{proof}
Let \(\mathcal E_\alpha\) denote the simultaneous coverage event from the theorem statement:
\[
\mathcal E_\alpha := \left\{\Delta_\pi(s_0)\ge L_\alpha(\pi,s_0)\ \text{for all } s_0\in\mathcal S_0\right\},
\]
with \(\Pr(\mathcal E_\alpha)\ge 1-\alpha\). We prove the claimed uniform lower bound on \(\mathcal S\) on this event.

Fix any arbitrary segment \(s\in\mathcal S\). Since \(\mathcal S_0\) is a \(\rho\)-cover of \(\mathcal S\), there exists at least one grid point \(s_0\in\mathcal S_0\) such that
\[
d(s,s_0)\le \rho.
\]
Because \(\Delta_\pi(\cdot)\) is \(L_s\)-Lipschitz, we have
\[
\Delta_\pi(s) \ge \Delta_\pi(s_0) - L_s\, d(s,s_0).
\]
Using \(d(s,s_0)\le \rho\), this implies
\[
\Delta_\pi(s) \ge \Delta_\pi(s_0)-L_s\rho.
\]
Now work on the event \(\mathcal E_\alpha\). There, every grid point lower bound is valid, so
\[
\Delta_\pi(s_0)\ge L_\alpha(\pi,s_0)
\qquad\text{for all } s_0\in\mathcal S_0.
\]
If \(\min_{s_0\in\mathcal S_0}L_\alpha(\pi,s_0)\ge \eta\), then in particular
\[
\Delta_\pi(s_0)\ge \eta
\qquad\text{for every } s_0\in\mathcal S_0.
\]
Substituting this into the previous display gives
\[
\Delta_\pi(s)\ge \eta - L_s\rho.
\]
Since \(s\in\mathcal S\) was arbitrary, the bound holds uniformly over all segments:
\[
\Delta_\pi(s)\ge \eta-L_s\rho
\qquad\text{for all } s\in\mathcal S.
\]

The certification claim is immediate. If \(\eta>L_s\rho\), then \(\eta-L_s\rho>0\), so \(\Delta_\pi(s)\ge 0\) for every segment \(s\), which means the policy is certified non-harmful on \(\mathcal S\). If \(\eta\ge L_s\rho+\gamma\) for some \(\gamma>0\), then
\[
\Delta_\pi(s)\ge \gamma
\qquad\text{for all } s\in\mathcal S,
\]
so the policy enjoys a uniform safety margin \(\gamma\) over the entire segment space.
\end{proof}

Now we move on to the proof of Proposition~\ref{prop:segment-sample-complexity}.

\begin{proof}
Let \(K:=|\mathcal S_0|\). For each grid segment \(s_0\in\mathcal S_0\), let \(\widehat{\Delta}_\pi(s_0)\) denote the empirical segment lift and let \(r_{n,s_0}\) be a confidence radius such that
\[
\Pr\!\left(
\Delta_\pi(s_0)\ge \widehat{\Delta}_\pi(s_0)-r_{n,s_0}
\ \text{for all } s_0\in\mathcal S_0
\right)\ge 1-\alpha.
\]
Because segment rewards are bounded in \([0,A]\), Hoeffding's inequality gives, for a fixed segment \(s_0\),
\[
\Pr\!\left(
\left|\widehat{\Delta}_\pi(s_0)-\Delta_\pi(s_0)\right|>r_{n,s_0}
\right)
\le 2\exp\!\left(-c\,\frac{n_{s_0}r_{n,s_0}^2}{A^2}\right)
\]
for a universal constant \(c>0\). Taking a union bound over the \(K\) grid segments, it is sufficient to choose
\[
r_{n,s_0}\asymp A\sqrt{\frac{\log(K/\alpha)}{n_{s_0}}}
\]
so that the simultaneous event holds with probability at least \(1-\alpha\).

Now suppose we want to certify \(\Delta_\pi(s)\ge 0\) for all \(s\in\mathcal S\). By Lemma~\ref{thm:segment-nonharm}, it is enough to ensure that every grid point satisfies
\[
L_\alpha(\pi,s_0)\ge \eta
\qquad\text{with}\qquad
\eta>L_s\rho.
\]
A sufficient empirical condition is
\[
\widehat{\Delta}_\pi(s_0)-r_{n,s_0}\ge \eta
\qquad\text{for all } s_0\in\mathcal S_0,
\]
which is equivalent to
\[
\widehat{\Delta}_\pi(s_0)\ge \eta + r_{n,s_0}
\qquad\text{for all } s_0\in\mathcal S_0.
\]
To make this possible with the stated confidence radius, it suffices that
\[
A\sqrt{\frac{\log(K/\alpha)}{n_{s_0}}}\lesssim \eta+L_s\rho.
\]
Rearranging yields
\[
n_{s_0}\gtrsim \frac{A^2\log(K/\alpha)}{(\eta+L_s\rho)^2}.
\]
Thus, if every observed segment has at least this many samples, then each grid point can be certified above the level needed to overcome the Lipschitz covering loss \(L_s\rho\), and therefore the entire segment space satisfies
\[
\Delta_\pi(s)\ge 0
\qquad\text{for all } s\in\mathcal S.
\]

This shows both the role of the union bound, which contributes the \(\log(K/\alpha)\) factor, and the role of the safety margin, which must absorb both statistical uncertainty and the covering error \(L_s\rho\).
\end{proof}

\section{Additional Experimental Results}
\label{app:additional-experiments}

This appendix reports supplementary diagnostics for the empirical claims in Section~\ref{sec:experiments}. The additional results are not separate experiments with a different agenda. They stress-test the same support-aware decision pipeline by varying the catalog size, the diagnostic boundary-window width, the localization level \(q\), the elimination tolerance, the segment grid, and the replay panel used for validation. Table~\ref{tab:theory-empirical-map} summarizes how the theoretical results are checked empirically.

\begin{table}[t]
\centering
\caption{\small Mapping between theoretical results and empirical checks.}
\label{tab:theory-empirical-map}
\small
\setlength{\tabcolsep}{3pt}
\renewcommand{\arraystretch}{1.10}
\begin{tabular}{%
>{\raggedright\arraybackslash}p{0.27\linewidth}
>{\raggedright\arraybackslash}p{0.43\linewidth}
>{\raggedright\arraybackslash}p{0.22\linewidth}}
\toprule
Theoretical result & Empirical check & Location \\
\midrule
Theorem~\ref{thm:replay-generalization} & Boundary-window and \(q\)-localized replay diagnostics estimate how quickly local threshold evidence accumulates and whether local ranking differs from aggregate replay ranking. & Fig.~\ref{fig:exp-support-resolution}; Fig.~\ref{fig:app-support-diagnostics}; Fig.~\ref{fig:app-q-localized-selection} \\
Corollary~\ref{cor:support-response} & Frozen season-three replay and response-gap discussion separate stable replay transfer from unresolved bidder-response risk. & Fig.~\ref{fig:exp-validation-readiness}; Appendix~\ref{app:transfer} \\
Proposition~\ref{prop:threshold-resolution} & Pairwise boundary-support diagnostics show that nearby reserve rules can have little effective local sample even in a large replay panel. & Fig.~\ref{fig:exp-support-resolution}; Fig.~\ref{fig:app-support-diagnostics} \\
Lower-bound elimination rule in Section~4.3 & Conservative shortlist construction compares upper bounds with the support-adjusted lower bound of the leader and eliminates dominated policies. & Fig.~\ref{fig:exp-shortlist}; Fig.~\ref{fig:app-decision-robustness} \\
Theorem~\ref{thm:reserve-pipeline} & The final decision object certifies P18, leaves P11 unresolved, and eliminates 17 dominated policies under simultaneous uncertainty and support gates. & Fig.~\ref{fig:exp-shortlist}; Appendix~\ref{app:elimination} \\
Proposition~\ref{prop:segment-sample-complexity} and Lemma~\ref{thm:segment-nonharm} & Segment diagnostics check non-harm across advertiser, exchange, region, inventory-category, and bid-gap strata. & Fig.~\ref{fig:exp-validation-readiness}; Appendix~\ref{app:segment-diagnostics} \\
Corollary~\ref{cor:replay-regret} & Boundary-window diagnostics examine localized replay support, while bootstrap replay diagnostics provide a complementary stability check for the full-sample replay winner. & Fig.~\ref{fig:exp-support-resolution}; Fig.~\ref{fig:app-replay-diagnostics} \\
\bottomrule
\end{tabular}
\end{table}

\subsection{Additional replay and replay-concentration diagnostics}
\label{app:replay-diagnostics}

\begin{figure}[t]
\centering
\includegraphics[width=0.98\linewidth]{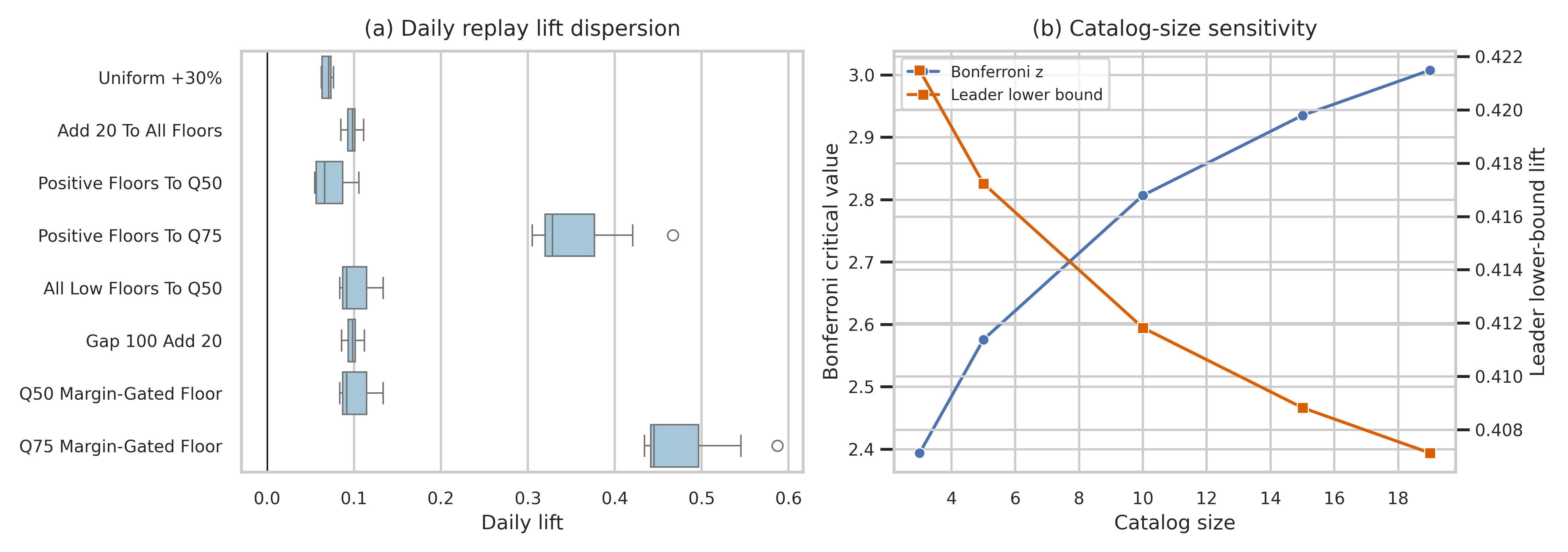}
\caption{\small Additional replay diagnostics. Panel (a) shows daily replay-lift dispersion for the leading policies. Panel (b) shows how the Bonferroni critical value and the leader's simultaneous lower bound change as the policy catalog grows.}
\label{fig:app-replay-diagnostics}
\end{figure}

Figure~\ref{fig:app-replay-diagnostics} provides complementary full-replay stability diagnostics for the lower-bound event \(\mathcal E_{\mathrm{lb}}\) and the finite-catalog decision object. The daily replay distributions show that P18 remains separated from the rest of the catalog across the seven season-two days. Bootstrap replay resampling selects P18 in all \(1000\) bootstrap draws, producing zero empirical replay regret relative to the full-sample replay winner. These diagnostics do not replace the localized support check in Section~\ref{subsec:exp-support}; instead, they show that the full replay winner is also stable under resampling of the logged panel.

The catalog-size panel shows the multiplicity cost of screening more policies. As the catalog size increases from 3 to 19, the Bonferroni critical value increases from 2.394 to 3.008. P18 remains the lower-bound leader throughout, but its certified lower-bound lift decreases from \(42.15\%\) to \(40.71\%\). This decline is expected. The point estimate is not deteriorating. The simultaneous lower bound becomes more conservative because the framework protects against false certification across a larger finite catalog.

\subsection{Additional support-resolution diagnostics}
\label{app:support-diagnostics}

\begin{figure}[t]
\centering
\includegraphics[width=0.98\linewidth]{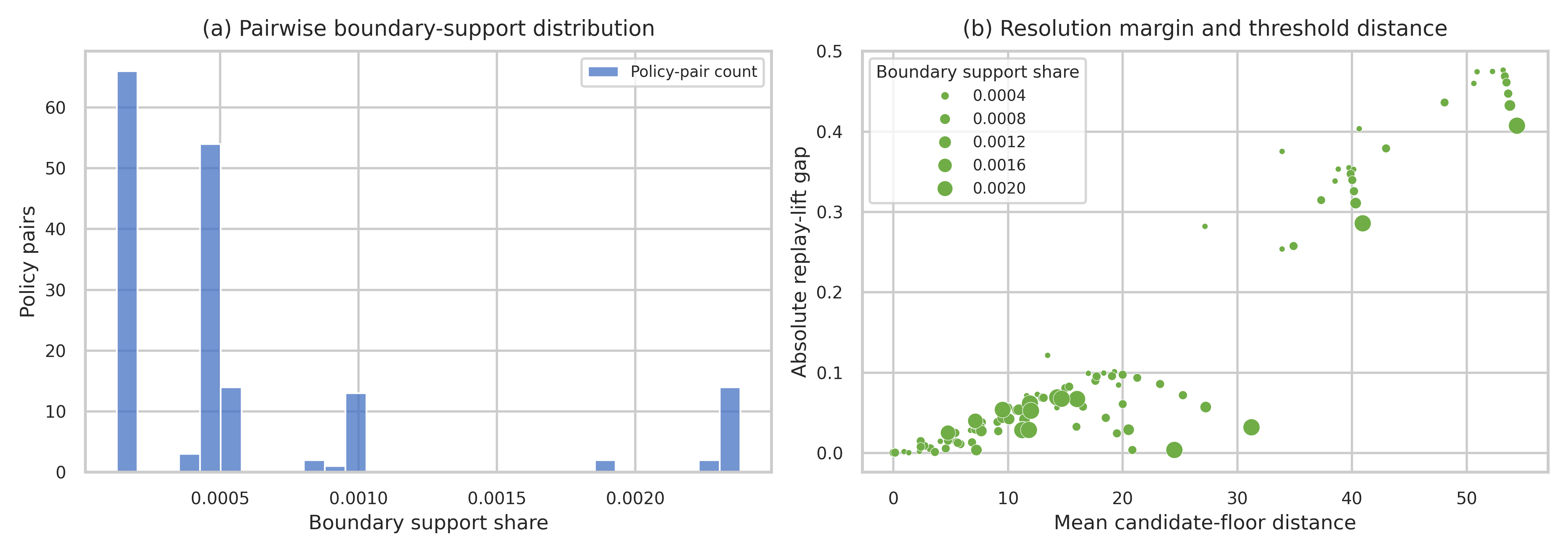}
\caption{\small Pairwise boundary-support diagnostics. Panel (a) reports the empirical distribution of pairwise boundary-support shares across policy pairs. Panel (b) relates mean candidate-floor distance to absolute replay-lift gaps, with marker size proportional to boundary support.}
\label{fig:app-support-diagnostics}
\end{figure}

The fine-grained boundary-window sweep extends Section~\ref{subsec:exp-support} using widths
\[
1,2,3,5,7.5,10,15,20,30,50,75,100,150,200.
\]
The monotone pattern from the main text remains. Boundary support grows as the diagnostic window widens, and support-adjusted lower bounds become less conservative when the effective boundary sample increases. Read in the direction of Theorem~\ref{thm:replay-generalization}, the same curve also describes the empirical inverse map from a desired evidence fraction \(q\) to the radius needed to collect that evidence. A steep accumulation curve means \(r_\pi(q)\) is small for many values of \(q\), while a flat curve indicates that the policy needs a wide threshold neighborhood before enough local evidence is available.

Figure~\ref{fig:app-support-diagnostics} studies pairwise boundary support directly. Across the 171 nonredundant policy pairs, the median boundary-support share is only \(0.046\%\), the minimum is \(0.0127\%\), and the maximum is \(0.238\%\). These values are small relative to the full season-two sample because most auction opportunities are far from the reserve thresholds where two candidate policies disagree. The right panel links this support limitation to resolution. Large floor-distance comparisons can have large replay gaps, but the relevant sample remains the local boundary sample rather than the full log. This is the empirical reason the theory uses \(m_{\tau,\tau'}\) rather than only \(n\).

\subsection{\texorpdfstring{\(q\)-localized replay selection}{q-localized replay selection}}
\label{app:q-localized}

The localized replay theorem motivates a diagnostic that ranks policies by their boundary-supported replay contribution. For each non-baseline policy \(\pi\) and localization level \(q\), we identify the smallest empirical radius \(r_\pi(q)\) that contains a fraction \(q\) of the floor-changing observations closest to the candidate reserve threshold. This is a contrast-local implementation of the theorem's population radius, with \(q\) normalized over observations where the candidate rule changes the logged floor so that the diagnostic focuses on the relevant decision boundary. We then compute a localized boundary lift using only observations in this boundary set. This quantity is not the same as aggregate replay lift. It measures local yield efficiency near the threshold, whereas aggregate replay measures total marketplace value under the fixed-bid replay contract.

\begin{figure}[t]
\centering
\includegraphics[width=0.98\linewidth]{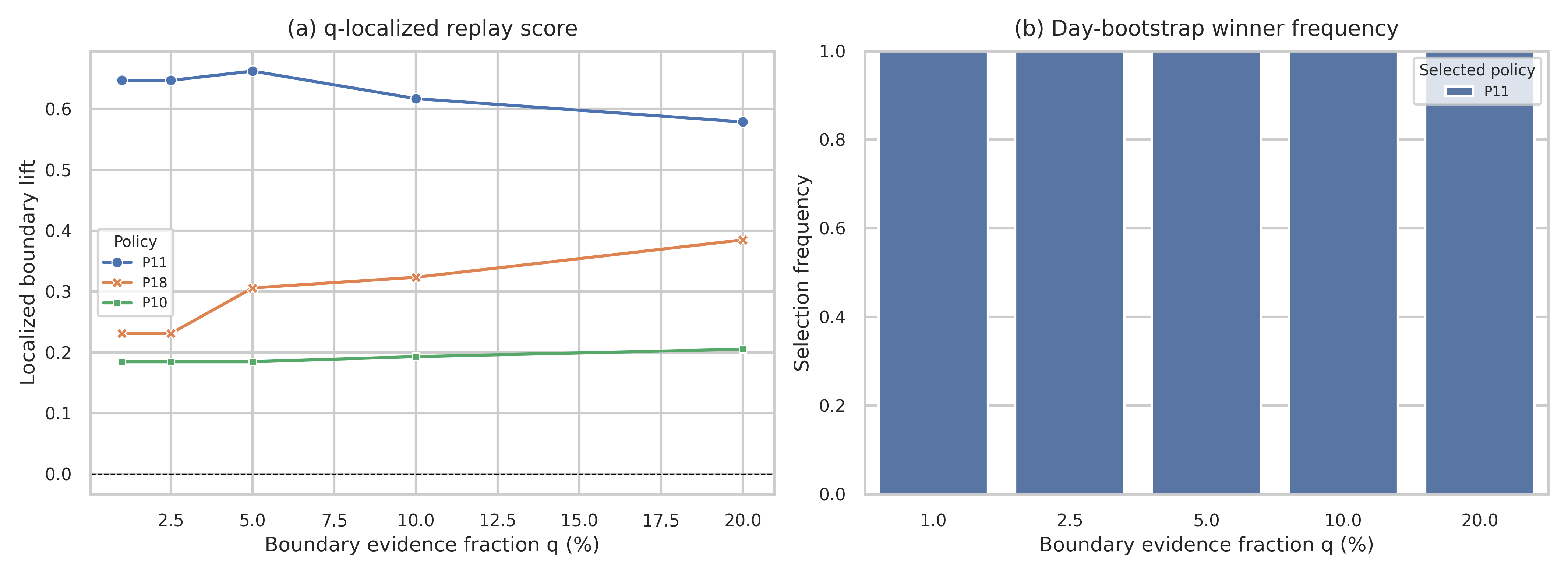}
\caption{\small \(q\)-localized replay selection. Panel (a) reports localized boundary lift over localization levels \(q\). Panel (b) reports day-bootstrap winner frequencies. P11 is the stable \(q\)-localized boundary-lift leader, while P18 remains the aggregate replay leader.}
\label{fig:app-q-localized-selection}
\end{figure}

Figure~\ref{fig:app-q-localized-selection} reports the resulting \(q\)-localized rankings and day-bootstrap stability. P11, Positive Floors To Q75, is the localized boundary-lift leader for every tested localization level \(q\in\{0.01,0.025,0.05,0.10,0.20\}\), with bootstrap selection frequency equal to \(1.00\) at each \(q\). P18, the Q75 margin-gated floor rule, remains second in the localized ranking but first in aggregate replay. The diagnostic therefore reveals a useful distinction between local threshold efficiency and aggregate policy value. P11 produces the strongest local gain near its threshold, while P18 produces the strongest total replay lift.

\begin{figure}[t]
\centering
\includegraphics[width=0.98\linewidth]{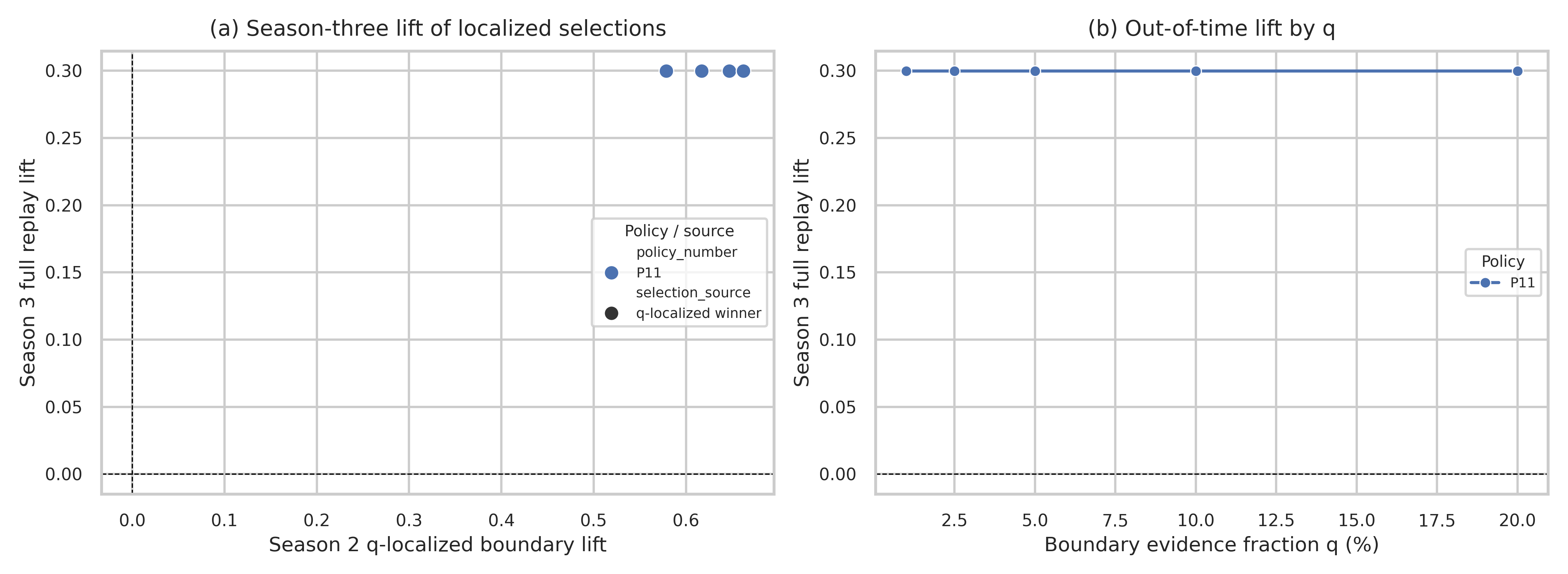}
\caption{\small Out-of-time transfer for \(q\)-localized selections. Panel (a) compares season-two localized boundary lift with season-three aggregate replay lift for policies selected by the \(q\)-localized rule. Panel (b) reports season-three aggregate replay lift across localization levels. The localized winner P11 transfers positively but does not exceed P18's aggregate season-three replay performance.}
\label{fig:app-q-localized-transfer}
\end{figure}

Figure~\ref{fig:app-q-localized-transfer} evaluates the out-of-time performance of the \(q\)-localized selections on season three. The localized winner P11 transfers positively, with a season-three aggregate replay lift of \(29.97\%\) and full retained-impression share. However, its aggregate lift remains below the season-three lift of P18 reported in the main validation analysis. Thus the \(q\)-localized diagnostic supports the paper's shortlist logic rather than overturning it. P11 is not eliminated because it is strongly supported locally. P18 remains the preferred validation candidate because it dominates on aggregate replay, conservative lower-bound ranking, and out-of-time transfer.

\subsection{Additional shortlist and elimination analyses}
\label{app:elimination}

\begin{figure}[t]
\centering
\includegraphics[width=0.98\linewidth]{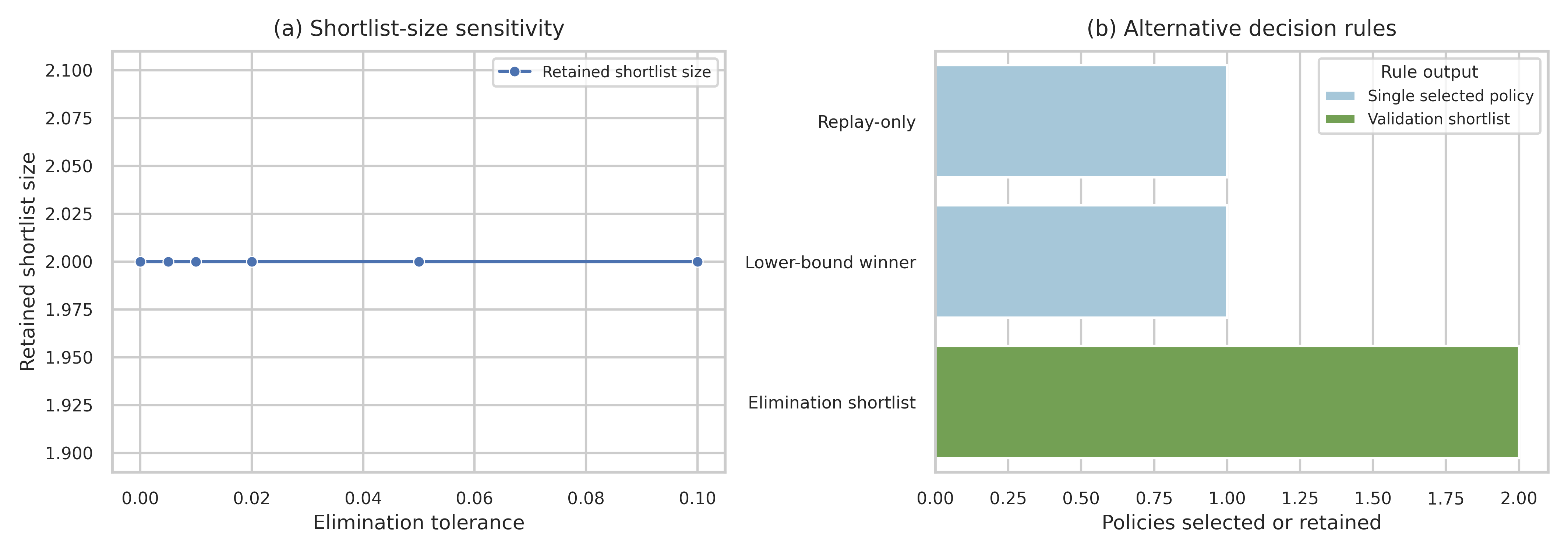}
\caption{\small Shortlist and decision-rule robustness. Panel (a) reports retained shortlist size as the elimination tolerance varies. Panel (b) compares simpler decision rules with the support-aware elimination shortlist.}
\label{fig:app-decision-robustness}
\end{figure}

Figure~\ref{fig:app-decision-robustness} shows that the two-policy shortlist is not an artifact of a single tolerance choice. As the elimination tolerance \(\rho\) varies from 0 to 0.10, the retained shortlist remains \(\{\text{P18},\text{P11}\}\). P18 is certified, P11 remains unresolved, and the other 17 policies remain dominated.

The decision-rule comparison clarifies what the framework adds beyond point-estimate ranking. Replay-only selection chooses P18. Lower-bound winner selection also chooses P18. The support-aware elimination rule retains P18 and P11, which is a more conservative output because P11's upper bound is not below P18's support-adjusted lower bound. The framework therefore does not merely confirm the replay winner. It also records the unresolved competitor that should remain visible in validation planning.

Dominated-policy diagnostics give the complementary view. Policies such as P16 and P8 have positive replay lifts, but their upper bounds are below the support-adjusted lower bound of P18. They are attractive relative to the logged baseline but not competitive with the certified lower-bound leader. This is the operational meaning of certified regret in Theorem~\ref{thm:reserve-pipeline}.

\subsection{Additional subgroup, transfer, and robustness diagnostics}
\label{app:segment-diagnostics}
\label{app:transfer}

The expanded subgroup analysis adds inventory-category and bid-gap strata to the advertiser, exchange, and region grid used in the main text. This creates 58 covered segments. All 58 have nonnegative lower confidence endpoints. The smallest lower endpoint is exactly zero and occurs in bid-gap buckets where the priority policy leaves the relevant observations unchanged. Thus, the expanded grid finds no negative subgroup certificate.

The coverage-radius sensitivity check applies a simple Lipschitz-style segment penalty. With no extra coverage penalty, all 58 segments certify non-harm. At coverage radius 0.25 and 0.5, 55 segments remain certified. At coverage radius 1.0, 54 remain certified, and at coverage radius 2.0, 53 remain certified. This pattern is useful because it shows where the empirical evidence is strongest. The observed segments are safe under the direct lower-endpoint calculation, while broader claims over unobserved nearby segments require stronger smoothness or more data.

The multiplicity-scaling diagnostic varies the number of simultaneously protected segments. The Bonferroni critical value increases with the number of segments, as expected, but all 58 observed segments remain certified over the tested range. The sparse-segment behavior diagnostic is empty under the current minimum-observation filter, which means no covered segment is both included in the analysis and assigned a negative lower endpoint.

The out-of-time transfer diagnostics also support the main-text validation-readiness conclusion. The season-two to season-three Spearman rank correlation is \(0.988\), and the top-five overlap is four policies. The response-gap sensitivity check uses the full-replay margin between P18 and P11, which is \(12.15\) percentage points in season two. Under the symmetric response-gap calculation, the P18 ranking is preserved when the pairwise response gap is below approximately \(6.08\) percentage points and is no longer guaranteed above that level. This sensitivity calculation is a planning diagnostic rather than a formal consequence of Corollary~\ref{cor:support-response}, which is stated for localized replay margins. It explains why the paper treats replay as a validation-screening device rather than a deployment claim.

Finally, validation-readiness stress tests vary the confidence level and the retained-impression support-penalty scale from Section~\ref{subsec:exp-shortlist}. Across the tested settings, P18 remains the leader, the shortlist remains size two, and the decision labels remain one certified policy, one unresolved policy, and 17 dominated policies. The stress tests therefore support the stability of the decision object, while the response-gap analysis preserves the central caution that live bidder response must be tested online.

\subsection{Implementation and reproducibility details}
\label{app:implementation}

\noindent \textbf{Policy catalog construction.}
Table~\ref{tab:policy-catalog} reports the complete reserve-policy catalog used in the experiments. The catalog is intentionally finite and operational. It contains the logged baseline, uniform percentage increases, uniform absolute increases, empirical quantile floors, and margin-gated rules that only raise floors when the logged bid-floor gap indicates room to move. The quantiles \(q_{25}\), \(q_{50}\), and \(q_{75}\) are computed from positive logged floors in the season-two development panel and then held fixed for all policy evaluation, including the frozen season-three transfer check.

\begin{table}[t]
\centering
\caption{\small Reserve-policy catalog evaluated in the experiments.}
\label{tab:policy-catalog}
\small
\setlength{\tabcolsep}{3pt}
\renewcommand{\arraystretch}{1.08}
\begin{tabular}{%
>{\raggedright\arraybackslash}p{0.07\linewidth}
>{\raggedright\arraybackslash}p{0.25\linewidth}
>{\raggedright\arraybackslash}p{0.18\linewidth}
>{\raggedright\arraybackslash}p{0.40\linewidth}}
\toprule
Policy & Reader-facing name & Family & Rule \\
\midrule
P0 & Logged Status Quo & Baseline & Use the logged floor \(f_i^0\). \\
P1 & Uniform \(+5\%\) & Uniform percent & Set \(f_i^\pi=1.05f_i^0\). \\
P2 & Uniform \(+10\%\) & Uniform percent & Set \(f_i^\pi=1.10f_i^0\). \\
P3 & Uniform \(+15\%\) & Uniform percent & Set \(f_i^\pi=1.15f_i^0\). \\
P4 & Uniform \(+20\%\) & Uniform percent & Set \(f_i^\pi=1.20f_i^0\). \\
P5 & Uniform \(+30\%\) & Uniform percent & Set \(f_i^\pi=1.30f_i^0\). \\
P6 & Add 5 To All Floors & Absolute increment & Set \(f_i^\pi=f_i^0+5\). \\
P7 & Add 10 To All Floors & Absolute increment & Set \(f_i^\pi=f_i^0+10\). \\
P8 & Add 20 To All Floors & Absolute increment & Set \(f_i^\pi=f_i^0+20\). \\
P9 & Positive Floors To Q25 & Positive-floor quantile & If \(f_i^0>0\), raise floors below \(q_{25}\) to \(q_{25}\). Zero floors stay zero. \\
P10 & Positive Floors To Q50 & Positive-floor quantile & If \(f_i^0>0\), raise floors below \(q_{50}\) to \(q_{50}\). Zero floors stay zero. \\
P11 & Positive Floors To Q75 & Positive-floor quantile & If \(f_i^0>0\), raise floors below \(q_{75}\) to \(q_{75}\). Zero floors stay zero. \\
P12 & All Low Floors To Q25 & All-floor quantile & Raise all floors below \(q_{25}\) to \(q_{25}\). \\
P13 & All Low Floors To Q50 & All-floor quantile & Raise all floors below \(q_{50}\) to \(q_{50}\). \\
P14 & Gap 25 Add 5 & Margin-gated increment & Add 5 when \(b_i-f_i^0\ge 25\), otherwise keep the logged floor. \\
P15 & Gap 50 Add 10 & Margin-gated increment & Add 10 when \(b_i-f_i^0\ge 50\), otherwise keep the logged floor. \\
P16 & Gap 100 Add 20 & Margin-gated increment & Add 20 when \(b_i-f_i^0\ge 100\), otherwise keep the logged floor. \\
P17 & Q50 Margin-Gated Floor & Hybrid quantile-margin & Raise to at least \(q_{50}\) when \(b_i-f_i^0\ge 50\), otherwise keep the logged floor. \\
P18 & Q75 Margin-Gated Floor & Hybrid quantile-margin & Raise to at least \(q_{75}\) when \(b_i-f_i^0\ge 100\), otherwise keep the logged floor. \\
\bottomrule
\end{tabular}
\end{table}

\noindent \textbf{Replay implementation details.}
For each logged opportunity, the replay code evaluates the candidate floor \(f_i^\pi\), retains the observation only if the logged impression filled and the logged bid clears the candidate floor, and assigns retained payment \(\max\{p_i,f_i^\pi\}\). The replay outcome is zero otherwise. Aggregate replay lift is computed relative to the logged-floor baseline. The conservative ranking in Section~\ref{subsec:exp-shortlist} uses daily replay variation with Bonferroni correction over the finite catalog and then subtracts the retained-impression support penalty. The threshold-resolution diagnostic in Section~\ref{subsec:exp-support} is separate: it keeps the full-panel replay lift fixed and varies the boundary-window count used in the inverse-square-root support penalty.

\noindent \textbf{Computational statistics.}
The reproducibility manifest records 19 policies, \(53{,}289{,}330\) season-two development opportunities, and \(10{,}566{,}743\) season-three validation opportunities. The notebook workflow is intentionally ordered rather than packaged as a single opaque command. Notebook 00 audits the archive, Notebook 01 builds the panels, Notebooks 02--04 reproduce the main Section~\ref{sec:experiments} experiments, Notebook 05 produces the Appendix B robustness artifacts, and Notebook 06 produces the \(q\)-localized replay selection and transfer diagnostics.

\noindent \textbf{Analysis map.}
The empirical analyses map directly to the theoretical pipeline. Notebook 02 tests conservative finite-catalog certification and elimination. Notebook 03 tests support-localized threshold resolution through boundary-window diagnostics. Notebook 04 tests frozen transfer and segment non-harm. Notebook 05 supplies robustness checks for replay concentration, support resolution, shortlist stability, subgroup safety, transfer, and implementation reproducibility. Notebook 06 tests whether \(q\)-localized boundary ranking changes the policy ordering and whether the locally selected policies transfer to season three.

\end{document}